\documentclass[10pt,letterpaper]{article}
\usepackage{tcolorbox}

\usepackage{amssymb}
\usepackage{amsthm}
\usepackage{amsmath}
\usepackage{graphicx}
\usepackage{fullpage}
\usepackage[margin=1in]{geometry}
\usepackage{enumerate}
\usepackage{algorithm,algorithmic}
\usepackage{tikz}
\makeatletter
\renewcommand*{\eqref}[1]{%
  \hyperref[{#1}]{\textup{\tagform@{\ref*{#1}}}}%
}
\makeatother
\usepackage[colorlinks=red,linkcolor=red]{hyperref}

\usepackage{latexsym,amssymb,amsmath,amsthm,graphics,graphicx,float,psfrag, epsfig, color, enumerate}

\usepackage{url,amssymb,amsmath,amsthm,amscd,paralist,bbm,wrapfig}
\usepackage{bm}
\usepackage{algorithm}
\usepackage{algorithmic} 

\usepackage{mathtools}

\usepackage[numbers]{natbib}
\usepackage{tikz}

\newtheorem{thm}{Theorem}
\newtheorem{lem}{Lemma}
\newtheorem{cor}{Corollary}
\newtheorem{defn}{Definition}
\newtheorem{claim}{Claim}
\newtheorem{prop}{Proposition}

\newtheorem{remark}{Remark}

\newcommand{\R}{\mathbb{R}}

\DeclareMathOperator*{\argmin}{arg\,min}

\DeclarePairedDelimiterX{\infdivx}[2]{(}{)}{%
  #1\;\delimsize\|\;#2%
}

\title{A Recovery Theory for Diffusion Priors: \\ Deterministic Analysis of the Implicit Prior Algorithm} 

\author{Oscar Leong\thanks{University of California, Los Angeles (email: \texttt{oleong@stat.ucla.edu})}
\and Yann Traonmilin\thanks{University of Bordeaux (email: \texttt{yann.traonmilin@u-bordeaux.fr})}}

\begin{document}

\maketitle

\begin{abstract}
Recovering high-dimensional signals from corrupted measurements is a central challenge in inverse problems. Recent advances in generative diffusion models have shown remarkable empirical success in providing strong data-driven priors, but rigorous recovery guarantees remain limited. In this work, we develop a theoretical framework for analyzing deterministic diffusion-based algorithms for inverse problems, focusing on a deterministic version of the algorithm proposed by Kadkhodaie \& Simoncelli \cite{kadkhodaie2021stochastic}. First, we show that when the underlying data distribution concentrates on a low-dimensional model set, the associated noise-convolved scores can be interpreted as time-varying projections onto such a set. This leads to interpreting previous algorithms using diffusion priors for inverse problems as generalized projected gradient descent methods with varying projections. When the sensing matrix satisfies a restricted isometry property over the model set, we can derive quantitative convergence rates that depend explicitly on the noise schedule. We apply our framework to two instructive data distributions: uniform distributions over low-dimensional compact, convex sets and low-rank Gaussian mixture models. In the latter setting, we can establish global convergence guarantees despite the nonconvexity of the underlying model set.
\end{abstract}

\section{Introduction} \label{sec:intro}

In this work, we consider solving linear inverse problems of the form $$y = A\hat{x}$$ where $A \in \R^{m \times d}$ is a linear operator and $\hat{x} \in \R^d$ is the underlying signal of interest. In many settings arising in the physical sciences, the above problem is ill-posed due to the compressive nature of the linear operator $A$, meaning that there are infinitely many signals that fit the given measurements. Hence one requires additional prior knowledge or regularizers to overcome the inherent ill-posedness.

Traditionally, hand-crafted priors based on notions such as sparsity \cite{Daubechiesetal04, Tao2006, Donoho2006, Tibshirani94}, low-rankness \cite{FazelThesis, CandesRecht09, Rechtetal10}, or smoothness \cite{Rudinetal92} have been widely integrated into variational approaches in solving inverse problems. With the development of modern machine learning techniques, we have seen a surge in more powerful, data-driven priors with strong empirical success. These include learning-based regularization penalties parameterized by deep neural networks \cite{lunz2018adversarial, Mukherjeeetal2021, Shumaylovetal2024, Goujonetal23} and methods based on plug-and-play denoisers \cite{venkatakrishnan2013plug, romano2017little}.

There has been a surge of interest in the use of generative models as priors in inverse problems due to their ability to generate realistic samples from a variety of high-dimensional data distributions. Early works along these lines have exploited GANs/VAEs \cite{bora2017compressed, Handetal2018, Handetal2024} and normalizing flows \cite{Asimetal20, ardizzone2018analyzing}. More recently, we have seen significant progress in natural image synthesis through the advent of diffusion models \cite{sohl2015deep, ho2020denoising, song2019generative, song2020score, nichol2021improved} and flow-matching models \cite{liu2022flow, lipman2022flow}. Diffusion models, for example, generate images by iteratively denoising noise samples with an iteration-dependent noise schedule. They are justified by an explicit relationship between the score of the noise-convolved data distribution and the minimum-mean-square-error (MMSE) denoising solution \cite{miyasawa1961empirical, vincent2011connection}. Once learned, these image priors have been successfully integrated into approaches for solving inverse problems \cite{Chungetal23, chung2022improving, feng2023score, kadkhodaie2021stochastic, wang2022zero, zhang2024flow, martin2024pnp, song2021solving, sun2024provable, xu2024provably, song2023solving}.

While the empirical performance of diffusion models for inverse problems is compelling, a rigorous recovery theory remains largely absent. Classical regularization theory provides quantitative guarantees on recovery error, stability, and robustness when using convex penalties or sparsity-inducing norms. These guarantees rely on an explicit low-dimensional model set $\Sigma \subset \mathbb{R}^d$ (e.g., a set of sparse vectors) that captures important structural properties of $\hat{x}$. By contrast, existing theoretical results for generative priors have primarily addressed representation error or sample complexity for GANs, VAEs, and normalizing flows, and most analyses of diffusion priors have focused on posterior sampling guarantees \cite{xu2024provably, sun2024provable} or restricted data models such as linear subspaces \cite{rout2023solving, rout2023theoretical}. As a result, we lack a systematic understanding of how diffusion priors enable recovery of ground-truth signals from compressive measurements. In particular, it remains unclear which classes of measurement operators permit recovery, how many observations are needed, or how the choice of noise schedule impacts convergence. This work develops a framework that addresses these questions.

A central difficulty is that diffusion priors are defined implicitly through denoisers that approximate the score of noise-convolved distributions, rather than through explicit convex penalties or geometric constraints. We show in this work that \textit{these denoisers can be viewed as approximate projection operators onto low-dimensional model sets}, which provides a direct connection between diffusion-based updates and generalized projected gradient methods. This perspective enables the analysis of recovery properties of diffusion priors, with guarantees that depend explicitly on the underlying data distribution and the noise schedule.

\paragraph{Contributions.}
This work develops a principled recovery theory for diffusion priors in linear inverse problems. Our main contributions are:
\begin{itemize}
    \item We show that for natural data distributions, noise-convolved scores associated with diffusion models can be interpreted as \emph{approximate projections} onto a low-dimensional model set $\Sigma$. Moreover, as the noise decreases, these approximate projections converge to a projection operator onto $\Sigma$.
    \item To exploit this perspective, we provide an in-depth analysis of a deterministic version of the sampling algorithm of \cite{kadkhodaie2021stochastic} for signal recovery from linear compressive measurements. We show that this algorithm naturally connects diffusion priors directly to generalized projected gradient descent methods.
    \item For the case of low-rank Gaussian mixture models (LR-GMMs), we prove that the limiting score recovers the true projection operator. Under a restricted isometry condition on the operator $A$, we establish global convergence of the iterates to the true signal. We show that, after a \textit{burn-in} period, the iterates converge at a rate that is a combination of a term exhibiting geometric decay and another term that depends on the noise schedule.
    \item We extend our analysis to other data distributions, including random sparse vectors and uniform distributions supported on convex bodies, where we prove that the limiting denoiser converges to the metric projection and quantify the approximation error.
\end{itemize} 
\section{Background and Related Work} \label{sec:background}

\paragraph{Diffusion priors in inverse problems.} Diffusion models have demonstrated remarkable empirical success in sampling from high-dimensional natural image distributions, which has led to their widespread use as unsupervised priors in inverse problems. Our focus here is not on the theory of generative modeling itself, but on theoretical results for their use as priors in inverse problems. A number of methods have been proposed; for a complete overview, see the survey \cite{daras2024survey}. Among the most influential is the Diffusion Posterior Sampling (DPS) algorithm \cite{Chungetal23}, which generates reconstructions by exploiting a tractable approximation to the time-dependent log-likelihood $\log p_t(y|x_t)$.  Extensions to latent diffusion models, such as \cite{rout2023solving}, show that when the ground-truth data lies in a low-dimensional subspace, the latent DPS algorithm asymptotically samples from the true posterior and recovers the ground-truth. Related theory for inpainting (Repaint \cite{lugmayr2022repaint}) has been established in \cite{rout2023theoretical}. Recent works \cite{wu2024principled, sun2024provable, xu2024provably} have leveraged diffusion priors in plug-and-play and MCMC-style algorithms, avoiding approximations used in prior works and established non-asymptotic stationarity guarantees. Particle-based approaches have also been developed, yielding asymptotically exact posterior sampling as number of particles increases \cite{dou2024diffusion, wu2023practical, cardoso2023monte}. Finally, \cite{bruna2024provable} transforms the original posterior sampling problem to a simpler one when given a denoising oracle, connecting the geometry of the posterior to properties of the inverse problem (e.g., conditioning of the measurement operator). In contrast to these works, our analysis departs from posterior sampling guarantees and develops a deterministic recovery theory, showing how diffusion priors can act as approximate projections and yield quantitative convergence guarantees for inverse problems.

\paragraph{Theory for learning diffusion models on data with low intrinsic dimension.} While our work is mainly focused on the use of diffusion models as priors for inverse problems, we would also like to highlight prior works that have analyzed learning diffusion models from data concentrated on low-dimensional sets and their sampling capabilities. For example, \cite{pmlr-v291-liang25a} study the iteration complexity of sampling from a distribution using DDIM and DDPM-type samplers, showing that the induced distribution has TV distance to the true distribution on the order of $\epsilon$ when the number of steps scales like $O(k / \epsilon)$ where $k$ captures the log-covering number of the support of the distribution. Under a different manifold hypothesis on the support of the distribution, the authors in \cite{pierret2024diffusion} exhibit convergence rates for the KL-divergence, showing that $O(k /\epsilon)$ iterations are needed to achieve KL-divergence on the order of $\epsilon$ where $k$ is the dimension of the manifold. On the sample complexity side, other works \cite{wang2024diffusion, chen2023score, achilli2025memorization, merger2025generalization} have also aimed to characterize the sample complexity of learning distributions with intrinsic low-dimensional structure, such as LR-GMMs or data supported on a manifold, via diffusion models. Our work complements this line of research by showing that when diffusion models are trained on data with low intrinsic dimension, they provably exploit the underlying geometry when used as priors for inverse problems via approximate projections onto the data's support, allowing us to establish rigorous sample complexity and recovery guarantees in this setting.

\paragraph{Solving ill-posed inverse problems with deep projective priors.}  In the vast choice of methods for the recovery of low-dimensional signals from compressive measurements \cite{foucart2013book}, the most relevant to our work are methods based on the projected gradient descent algorithm (PGD). Called iterative hard thresholding for sparse recovery \cite{blumensath2010normalized}, generalizations of this algorithm  yield stable recovery when the number of measurements is of the order of the dimensionality of the model set (up to logarithmic terms, typically through a restricted isometry condition) \cite{golbabaee2018inexact,bahmani2016learning}. The benefit of PGD is that only a projection on the model set is sufficient to perform iterations. This makes the PGD framework a natural tool to design algorithms using learned \textit{deep projective priors}. In particular, plug-and-play methods rely on a general purpose denoiser (e.g. implemented with a deep neural network (DNN)) to perform a generalized projection algorithm. In particular, the ``proximal gradient descent" variant of the algorithm linearly converges to a stable solution under a restricted isometry condition on measurements and a Lipschitz condition on the projection~\cite{liu2021recovery}. This was improved to a \emph{restricted Lipschitz} condition leading to a common interpretation of sparse recovery and plug-and-play PGD-like methods~\cite{traonmilin2024towards}. The restricted Lipschitz constant is a good indicator of the performance of PGD, as was shown by learning methods that improve this constant \cite{joundi2025stochastic}. Many other variations of plug-and-play methods exist \cite{romano2017little,reehorst2018regularization} with their dedicated study of convergence \cite{cohen2021regularization,hurault2022gradient} but are out of the scope of this paper that makes the specific connection between the diffusion approach from \cite{kadkhodaie2021stochastic} and a generalization of PGD.


\subsection{Notation}

Consider a set $\Sigma \subset \R^d$. We say that a (potentially set-valued) map $P$ is a (generalized) projection onto $\Sigma$ if for every $x \in \R^d$, $P(x) \subset \Sigma$. We reserve the notation $P^{\perp}_{\Sigma}$ for the (potentially set-valued) orthogonal projection (or metric projection) onto $\Sigma$: $$P_{\Sigma}^{\perp}(x):= \mathrm{proj}_{\Sigma}(x):= \argmin_{z \in \Sigma} \|x - z\|_2.$$ Given a function $f : \R^d \rightarrow \R^d$, we let $\|f\|_{\mathrm{op}}$ denote the operator norm induced by the Euclidean norm $$\|f\|_{\mathrm{op}} := \sup_{\|x\|_2 = 1} \|f(x)\|_2.$$ We use the notation $a\lesssim b$ to refer to when there exists a positive absolute constant $C$ such that $a \le Cb.$  The normal distribution of mean $\xi$ and covariance matrix $\Theta$ is denoted $\mathcal{N}(\xi,\Theta)$ and its density evaluated at $x \in \mathbb{R}^d$ is denoted $\mathcal{N}(x;\xi,\Theta)$. For distributions $P$ and $Q$, let $P*Q$ denote their convolution, which corresponds to the distribution of $z = u + v$ where $u \sim P$ and $v \sim Q$.
\section{Problem Setup} \label{sec:problem-setup}

Suppose that we are given linear measurements $y = A\hat{x}$ of an underlying signal $\hat{x}$. In this work, we will assume that $\hat{x}$ lies in a model set $\Sigma \subset \R^d$ with low intrinsic dimension. We limit ourselves to the noiseless case to focus on the underlying low-dimensional geometry of diffusion methods for inverse problems.

To address this problem, we take a Bayesian approach in which we assume access to a prior distribution $p$ that concentrates on the low-dimensional model set $\Sigma \subset \R^d$. Note that score-based diffusion priors obtain estimates of the score of the noise-convolved density $p_{\sigma} := p * \mathcal{N}(0,\sigma^2I)$ for various values of $\sigma > 0$. Recall that via Tweedie's formula \cite{miyasawa1961empirical}, there is a connection between the score of the noise-convolved density $p_{\sigma}$ and the MMSE denoiser $D_{\sigma}(x):= \mathbb{E}[x_0 |x_0+\sigma z=x]$ given by $$\sigma^2 \nabla \log p_{\sigma}(x) = D_{\sigma}(x)-x.$$ Conditioned on linear measurements, one can compute the conditional score to estimate $\hat{x}$ from linear measurements \cite{kadkhodaie2021stochastic}: $$\sigma^2\nabla \log p_{\sigma}(x|y)  = A^T(y - Ax) + (I - A^TA)\sigma^2\nabla \log p_{\sigma}(x).$$ By choosing a noise schedule $(\sigma_n)_{n \geq 0}$ that changes in a coarse-to-fine manner and adding Gaussian noise, we can use the above gradient steps to derive a stochastic estimate of $\hat{x}$.

In this work, we will consider a deterministic version of the algorithm presented in \cite{kadkhodaie2021stochastic} and obtain recovery guarantees for $\hat{x}$. In particular, for a noise schedule $(\sigma_n)_{n \geq 0}$ that decays $\sigma_n \rightarrow 0$ as $n \rightarrow \infty$, we will consider iterates of the form
\begin{align}\label{eq:iterations1}
x_{n+1} = x_n - \sigma^2_n \nabla \log p_{\sigma_n}(x|y) = x_n - (d_n + g_n)
\end{align}
where \begin{align*}
    d_n & := A^T(Ax_n - y)\ \text{and} \\
    g_n & := -\sigma_n^2(I-A^TA)\nabla \log p_{\sigma_n}(x_n) \\
    & = (I-A^TA)(x_n-D_{\sigma_n}(x_n)).
\end{align*}
Here $d_n$ is a direction that encourages data fit while $g_n$ is a direction that encourages high likelihood under the distribution $p_{\sigma_n}$. These are precisely the iterations considered in \cite{kadkhodaie2021stochastic} with no noise and the step size set to the limiting choice of $\lim_{n\rightarrow\infty} h_n = 1.$ Our goal is to establish conditions under which the sequence of iterates $x_n$ converges to $\hat{x}$ and to establish quantitative convergence rates.
\section{Theoretical Framework} \label{sec:theory-framework}

The key insight of this work is that for many distributions concentrated on a low-dimensional model set $\Sigma$, the above iterations can be viewed as a variant of projected gradient methods where the projection operator is time-varying (i.e. iteration dependent). In particular, note that the prior direction is given by \begin{align*}
    -\sigma_n^2\nabla \log p_{\sigma_n}(x) & = x - D_{\sigma_n}(x) =: x - P^n(x)
\end{align*} where \begin{align} \label{eq:appx-projection-op}
    P^n(x):= x + \sigma_n^2 \nabla \log p_{\sigma_n}(x)
\end{align} is an iteration-dependent operator $P^n : \R^d \rightarrow \R^d$ that performs an approximate projection onto the model set $\Sigma$. The fundamental idea is that as $n \rightarrow \infty$, $P^n$ behaves like a true projection operator $P_{\Sigma}$ onto the model set $\Sigma$ in certain cases and the iterations \eqref{eq:iterations1} can be written as a generalized projected gradient descent with varying projections (GPGD-VP): algebraic manipulations yield that the iterates \eqref{eq:iterations1} can be written as 
\begin{align}\label{eq:iterations2}
x_{n+1} & = x_n - A^T(Ax_n- y) - (I-A^TA)(x_n-P^n(x_n)) \nonumber\\
&=  P^n(x_n) - A^T(A P^n(x_n) - y ).
\end{align} 
Note that we consider iterates before ``projecting" with $P^n$ (PGD is often presented as iterates \emph{after} projection). 
We will discuss several examples of data distributions and their corresponding projection operators in the following sections. In our analysis, we will consider the more general sequence of GPGD-VP iterates with step size $\mu > 0$: \begin{align}
    x_{n+1} = P^n(x_n) -  \mu A^T(A P^n(x_n) - y ). \label{eq:main-iterations}
\end{align} 

\subsection{Preliminaries}

For our theoretical results, we will assume certain properties of the underlying inverse problem, model set, and projection operators to obtain convergence guarantees. 

We will begin with the following definition of a restricted isometry constant, which is now classical in the context of signal recovery \cite{Tao2006, Donoho2006}.

\begin{defn}\label{def:RIC}
The operator $B$  has restricted isometry constant (RIC) $\delta <1$ on the secant set $\Sigma-\Sigma =\{x_1-x_2 : x_1,x_2 \in \Sigma \}$ if for all $x_1,x_2 \in \Sigma$,

\begin{equation}
 \|(B-I)(x_1-x_2)\|_2\leq \delta \|x_1-x_2\|_2.\nonumber
\end{equation}
We write $\delta_\Sigma(B)$ the  smallest admissible restricted isometry constant (RIC). In the following we will only consider the RIC of $B= \mu A^TA$. 
\end{defn}

Note that a lower restricted isometry property is necessary for the identification of a low-dimensional model set. With appropriate normalization $\mu$, this property is typically verified for a number of random Gaussian measurements $m \geq O(s \log(N/s))$ with high probability when $\Sigma = \Sigma_s$ the set of $s$-sparse vectors. It is also possible to construct compressive random measurements that satisfy the RIP for any low-dimensional model set, with a sample complexity on the order of the dimensionality of the model set (up to logarithmic factors) \cite{puy2017recipes}.

We now turn to properties of the projection operator. We will assume that our approximate projection operators $P^n$ converge (in a suitable sense) to a limiting projection operator $P_{\Sigma}$ that satisfies the following Lipschitz property with respect to the model set $\Sigma$. In \cite{traonmilin2024towards}, the restricted Lipschitz property was shown to guarantee the convergence of generalized projected gradient descent with fixed generalized projection. 

\begin{defn}[Restricted Lipschitz property]\label{def:lip_const}
Let $P$ be a generalized projection.  Then $P$ has the restricted $\beta$-Lipschitz property with respect to $\Sigma$ if for all  $z \in \mathbb{R}^d, x \in \Sigma, y \in P(z)$ we have
\begin{equation}
\begin{split}
 \|y-x\|_2 \leq \beta \|z-x\|_2.\\
 \end{split}\nonumber
\end{equation}
We define $\beta_{\Sigma}(P)$ to be the smallest $\beta$ such that $P$ has the restricted $\beta$-Lipschitz property.
\end{defn} Note that in general, we have that $\beta \geq 1$ with equality $\beta = 1$ in the case when $\Sigma$ is a closed, convex set and $P$ is the orthogonal projection $P_\Sigma^\perp$ onto $\Sigma$. For generic sets $\Sigma$, when the (set-valued) orthogonal projection  defined as $P_\Sigma^\perp(x)= \arg\min_{u\in\Sigma} \|u-x\|_2$ is non-empty in $\mathbb{R}^d$, it has restricted Lipschitz property $\beta_\Sigma(P_\Sigma^\perp) \leq 2$. When $\Sigma =\{x \in \R^d : \|x\|_0 \leq s\}$ is the set of $s$-sparse vectors, it was shown (see \cite[Theorem 3.2]{traonmilin2024towards}) that $\beta_\Sigma$ has a nearly optimal restricted Lipschitz constant $\beta_\Sigma(P_\Sigma^\perp) \leq \sqrt{\frac{3 +\sqrt{5}}{2}}   \approx 1.618 $.

\subsection{General results} \label{sec:general-conv-results}
We now turn towards our main results about iterations of the form \eqref{eq:main-iterations} in recovering $\hat{x}$ from $y = A\hat{x}$. Our first main result is a general bound for the iterates. It shows that when the projection errors converge, we can obtain the following upper bound:

\begin{thm}[Bound on iterates of GPGD-VP] \label{th:conv_var_proj}
  Suppose $P_\Sigma$ is restricted $\beta$-Lipschitz.  Let $\delta : =\delta (\mu A^TA)$. Consider the iterations \eqref{eq:main-iterations} with $\hat{x} \in \Sigma$ and assume the sequence $u_n:=\|P^n(x_n)-P_\Sigma(x_n)\|_2$ satisfies $\lim_{n\to\infty} u_n= 0$. Then if $\delta \beta < 1$, there exists a constant $C$ depending on $\hat{x}, x_0, \mu, A, \beta, (u_n)_{n\geq0}$ such that   
\begin{equation} \label{eq:general-iterates-bound-GPGD-VP}
 \|x_{n} -\hat{x}\|_2   \leq  C( ( \delta \beta)^{n/2}   + \max_{l=\lfloor n/2\rfloor,n}\|P^{l}(x_l) -P_\Sigma(x_l) \|_2).\\
 \end{equation}
\end{thm}

Depending on the behavior of $\|P^n-P_\Sigma\|_{\mathrm{op}}$, we can guarantee fast convergence or only (potentially slow) convergence as described in the two following corollaries. 

\begin{cor}[Slow convergence]\label{cor:slow_convergence} Under the hypotheses of Theorem~\ref{th:conv_var_proj}, we have
\begin{equation}
 \lim_{n\to\infty} \|x_{n} -\hat{x}\|_2 =  0. \nonumber
 \end{equation}
\end{cor}
\begin{cor}[Linear convergence]\label{cor:linear_convergence}  Suppose $P_\Sigma$ is restricted $\beta$-Lipschitz.  Let $\delta : =\delta (\mu A^TA)$. Consider the iterations \eqref{eq:main-iterations} with $\hat{x} \in \Sigma$ and assume  $\|P^n -P_\Sigma\|_{\mathrm{op}} \leq O(e^{-cn})$ with $c>0$. Then, there exist $C,D \geq 0$ such that
\begin{equation}
 \|x_{n} -\hat{x}\|_2   \leq  C  (\delta \beta)^{n/2} +  D e^{-\frac{c}{2}n}. \nonumber
 \end{equation}
\end{cor}

We will now analyze two examples of general distributions for which we can obtain recovery guarantees. The first considers data distributions concentrated on a low-dimensional convex set, while the latter focuses on the more challenging setting of data concentrated on a nonconvex union-of-subspaces.

\subsection{Application to uniform measures on convex sets} \label{sec:convex-sets}

 Consider a measure $p$ that is uniform on a compact, convex set $\Sigma \subset \R^d$. In this case, one can show that for a given noise schedule $(\sigma_n)_{n\geq 0}$, the iteration-dependent operator $P^n = D_{\sigma_n}$ in \eqref{eq:appx-projection-op} converges uniformly to the metric projection $P_{\Sigma}^{\perp}(x):=\argmin_{z\in\Sigma}\|x-z\|_2$: 
 $$\lim_{n\to\infty} \|P^n - P_{\Sigma}^{\perp}\|_{\mathrm{op}}  =0.$$ More specifically, one can obtain the following uniform error bound between the time-varying projection and the true projection operator onto $\Sigma$. Note that we allow for convex sets with low intrinsic dimension, which is captured by the dimensionality of its affine hull $\mathrm{aff}(\Sigma)$. Please see the Appendix (Section \ref{sec:unif-measure-convex-set}) for a more detailed discussion and proof of this result.

\begin{prop} \label{prop:uniform-convex}
    Suppose $\Sigma \subset \R^d$ is a compact, convex set containing the origin with intrinsic dimension $\mathrm{dim}(\mathrm{aff}(\Sigma))=s\leq d$ and let $p$ be the uniform measure on $\Sigma$. Then there exists a positive absolute constant $C$ such that the following holds: $$\|P^n - P_{\Sigma}^{\perp}\|_{\mathrm{op}} \leq C\sqrt{\left(\frac{s^2R_{\Sigma} \lambda_s(\Sigma)}{\kappa_{\Sigma,s}} \right)\cdot \sigma_n^2 \log(1/\sigma_n)}$$ where $R_{\Sigma}:= \max_{z \in \Sigma}\|z\|_2$, $\lambda_s(\cdot)$ is the $s$-dimensional Hausdorff measure restricted to $\mathrm{aff}(\Sigma)$, and $\kappa_{\Sigma,s}$ depends on the volume of the $s$-dimensional Euclidean ball and inner and outer width of $\Sigma$ in $\mathrm{aff}(\Sigma)$ (see \eqref{eq:kappa-def} in the Appendix).
\end{prop}

As a direct consequence, combining Proposition \ref{prop:uniform-convex} with Theorem \ref{th:conv_var_proj} and noting $P_{\Sigma}^{\perp}$ is $\beta=1$-restricted Lipschitz in this setting, immediately yields the following convergence guarantee. 
\begin{cor} \label{cor:convex-convergence}
    Suppose $\hat{x} \in \Sigma$ where $\Sigma$ is a compact, convex set containing the origin and let $p$ be the uniform measure on $\Sigma$. Consider the iterations \eqref{eq:main-iterations}. Then if $\delta :=\delta(\mu A^TA) < 1$, there exists a constant $C$ such that  \begin{align*}
        \|x_n-\hat{x}\|_2 \leq C\left(\delta^{n/2} + \sqrt{\sigma_n^2\log(1/\sigma_n)}\right).
    \end{align*}
\end{cor}
We observe that a geometric noise schedule is needed if we want to obtain a geometric convergence rate. Indeed, one can show that if $\sigma_n = \sigma_0 q^n$ for some $q \in (0,1)$, the rate is on the order of $$\delta^{n/2} + \sqrt{n}q ^n.$$ Choosing any $r \in (q,1)$, one further has $\sqrt{n}\leq (r/q)^n$ for $n$ sufficiently large so that $q^n\sqrt{n} \leq r^n$. This gives a rate of the following form for all $n$ sufficiently large: \begin{align*}
    \|x_n-\hat{x}\|_2 \leq C' \gamma^n,\ \text{where}\ \gamma:= \max\left\{\sqrt{\delta}, r\right\} < 1.
\end{align*} For nonconvex, low-dimensional sets associated with the LR-GMM in the next Section, we also observe the impact of the noise schedule on the convergence rate of GPGD-VP.

\subsection{Application to Low-Rank Gaussian Mixture Models (LR-GMMs)} \label{sec:LR-GMM-theory}

We now provide an example of a distribution $\hat{x} \sim p$ that is  expressive, concentrates on a nonconvex low-dimensional model set $\Sigma$  and whose noise-corrupted score can be interpreted in our theoretical framework. 
Consider a distribution $p$  given by a mixture of $K$ low-rank (degenerate) Gaussian distributions $$p = \sum_{k=1}^K \pi_k \mathcal{N}(0,U_kU_k^T)$$ where each $U_k \in \R^{d \times r_k}$ has rank $r_k \leq d$ and $\sum_{k=1}^K \pi_k = 1$ Here, we assume that each $U_k$ has orthonormal columns. In this example, the low-dimensional model set is given by the union of subspaces $E_k := \mathrm{Im}(U_k)$,  $\Sigma := \bigcup_{k=1}^K E_k$.

Gaussian mixture models are an expressive class of distributions that are both theoretically interesting and practically useful. Indeed, it has been shown that Gaussian mixture models can approximate smooth densities~\cite{dalal1983approximating}. In particular,  Low-Rank GMMs (LR-GMM), have been proposed to approximate low-dimensional data~\cite{chen2021learning} and have proved their usefulness in the context of imaging inverse problems~\cite{zoran2011learning,houdard2018high}.

Prior results have shown that one can analytically characterize the noise-convolved score for the LR-GMM. 
Using results in \cite{wang2024diffusion} (see, e.g., Proposition 1), one can show that for $\sigma_n>0$, the noise-convolved score is given by
\begin{align*}
       & \sigma_n^2 \nabla \log p_{\sigma_n}(x) \\
       & = \frac{1}{1+\sigma_n^2}\frac{\sum_{k}\pi_{k}\mathcal{N}(x; 0, U_kU_{k}^T + \sigma_n^2I)U_kU_k^Tx}{\sum_{l}\pi_{l}\mathcal{N}(x; 0, U_l U_{l}^T + \sigma_n^2I)} - x \\
       & =: P^n(x) - x.
    \end{align*} Letting $\omega_k(x,t):= \frac{\pi_k \mathcal{N}(x;0,U_kU_k^T+tI)}{\sum_{l}\pi_{l} \mathcal{N}(x;0,U_{l}U_{l}^T+tI)}$ and $P^{\perp}_{E_k}(x):= U_kU_k^Tx$, one can see that the operator $P^n$ is the following weighted sum of orthogonal projections onto each linear subspace $E_k$: 
    
    \begin{equation}\label{def:Pn_GMM}
    P^n(x) = \frac{1}{1+\sigma_n^2}\sum_{k=1}^K\omega_k(x, \sigma_n^2)P^{\perp}_{E_k}(x).\end{equation}

    \paragraph{Computation of the limiting projection.} Note that in simple cases, it is clear that the approximate projections converge uniformly with a decreasing noise schedule $\sigma_n \rightarrow 0$ and overall convergence is easily obtained. For example, in the case of $k = 1$, we have $\Sigma = \mathrm{Im}(U_1)$. Moreover, $P^n(x) = \frac{1}{1 + \sigma_n^2} U_1U_1^Tx = : \frac{1}{1+\sigma_n^2}P_{\Sigma}^{\perp}(x)$. Finally, since $\sigma_n^2 \rightarrow 0$ as $n \rightarrow \infty$ we have  \begin{align*}
        \|P^n - P_{\Sigma}^\perp\|_{\mathrm{op}} & =  \sup_{\|x\|_2=1}\|P^n(x) - P_{\Sigma}^\perp(x)\|_2 \\
        & =\left|\frac{1}{1+\sigma_n^2} - 1\right| \\
        & = O(\sigma_n^2) \xrightarrow[n \to \infty]{} 0.
    \end{align*}  For the more general case in \eqref{def:Pn_GMM}, we first aim to characterize what the value of the limiting projection would be. 
    
\begin{lem}\label{lem:limit_proj_GMM}
Let $\Sigma = \bigcup_{k=1}^K E_k$. Then for $x \in \mathbb{R}^d$, we have  
\begin{itemize}
 \item if $P_\Sigma^\perp(x)$ is single-valued, 
 \begin{equation}
 \lim_{n\to\infty} P^{n}(x)=  P_\Sigma^\perp(x) ; \nonumber
 \end{equation}
 \item if there is $r$ such that $\dim(E_l) = r$, and $x$ is such that $P_\Sigma^\perp(x) = \{u_l: l \in L\}$ is multi-valued (where $L \subset  \{1, \ldots,K\}$), then 

\begin{equation}
    \lim_{n\to\infty} P^{n}(x) = \frac{\sum_{l \in L } \pi_l u_l}{\sum_{l \in L } \pi_l}. \nonumber
\end{equation}
 
 \end{itemize}
\end{lem}
While given in the discrete $t_n$ case, this result is valid for continuous $t \to 0$.  In the following, we consider the GMM case with $K$ subspaces $E_1, \ldots, E_K$ with equal dimension $r$ with $t=t_n =\sigma_n^2 \to_{n \to \infty} 0$. 

We note that the restricted Lipschitz constant of the orthogonal projection plays an important role in our convergence results. For a union-of-subspaces, the geometry and relative orientation of the subspaces affect the restricted Lipschitz constant of the orthogonal projection $P_\Sigma^\perp$. Specifically, the orthogonal projection onto a finite union-of-subspaces may admit a restricted Lipschitz constant arbitrarily close to the upper bound $\beta =2$, as established for orthogonal projections onto nonconvex sets (see \cite[Remark 3.1]{traonmilin2024towards}). Thus, the precise value of $\beta$ in our setting reflects the geometry of the underlying subspaces.

\paragraph{Challenge.} As discussed in the previous section, if we have uniform convergence of the projection operators $P^n$ to $P_{\Sigma}$, then convergence of the entire scheme can be guaranteed. However, this may not always be the case for general distributions. In particular, for the LR-GMM setting, uniform control over $\|P^n(x) - P_{\Sigma}^\perp(x)\|_2$ cannot be expected, as the approximation error increases for points $x$ near the \textit{frontier} $F$ between subspaces, i.e., the set of points for which $x$ is equidistant to multiple subspaces. Hence a more fine-grained analysis needs to be done for this special case. Nevertheless, one can show that the iterates are eventually repelled away from such frontiers, which enables global convergence guarantees. At a high-level, we guarantee that after a given \textit{burn-in} time period, the iterates $x_n$ exhibit convergence to $\hat{x} \in \Sigma$ (with the right noise schedule) under the condition that $\delta\beta <1$. The convergence rate is a combination of a geometric convergence rate and an additional term that depends on the choice of noise schedule.

\begin{thm}[Global convergence] \label{th:global_convergence_gmm}
 Let $\delta:= \delta(\mu A^TA)$ and $\beta: = \beta_\Sigma(P_\Sigma^\perp)$.  Let $\hat{x} \in \Sigma  = \bigcup_{k=1}^K E_k $ such that there exists a unique $k \in \{1,\ldots,K\}$, such that $\hat{x} \in E_k$. Suppose $\delta \beta <1$ and $\sigma_n \to 0$. Then, there are $c,C>0$, $n_0$ such that for $n\geq n_0$
 \begin{align*}
 \|x_{n} -\hat{x}\|_2 \leq C\left( (\delta \beta)^{n/2}   + \max_{l=\lfloor n/2\rfloor,n}\exp\left( - \frac{c}{\sigma_l^2} \right)  +\sigma_l^2\right).
 \end{align*}
\end{thm}

\paragraph{Proof sketch.} The proof relies on establishing two properties. First, we show that away from the frontier between subspaces, the error between projection operators $\|P^n (x_n)- P_{\Sigma}(x_n)\|_2$ decays exponentially in $\eta/\sigma_n^{2}$, where $\eta > 0$ controls the distance to the frontier. Second, we establish that if iterates approach the frontier, the update rule introduces a repulsive effect that eventually drives the iterates into a stable region associated with the true subspace of $\hat{x}$. Within this neighborhood, iterates stay within a ball of the true signal $\hat{x}$ (see the definition of $\Omega_{\eta}$ and Theorem \ref{th:local_convergence} in Section \ref{sec:local-convergence-appx} of the Appendix). Establishing this result is the main technical difficulty of the proof. Combining these two ingredients yields the stated convergence rate. Similar results are typical in the literature of sparse recovery: once the subspace supporting the unknown signal is identified, linear convergence is observed (see e.g., \cite{liang2018local}). A detailed local convergence analysis, which quantifies the basin of attraction and proves stability away from the frontier, is provided in the Appendix.

Note that the requirement that $\hat{x}$ strictly belongs to a single subspace $E_k$ holds with probability $1$ under $p$, since the intersection of two distinct subspaces has measure zero with respect to $p$. Similarly to convergence guarantees with convex model sets, Theorem~\ref{th:global_convergence_gmm} provides a bound on the iterates that depends on both a geometric term and the noise schedule. If the noise schedule follows a geometric rate, the same argument as  in Corollary~\ref{cor:linear_convergence} yields geometric convergence.

\begin{remark}[Random sparse vectors] The LR-GMM framework can also provide an analysis for other classes of signals of interest, such as sparse vectors. In particular, consider the following generative model for random sparse vectors: we say that $x \sim p_s$ if and only if $x$ is generated by taking a subset $S \subset [d]$ of size $|S|=s$ uniformly and setting $x_i \sim \mathcal{N}(0,1)$ for $i \in [S]$ and $x_i = 0$ otherwise. An equivalent way of describing this distribution is with a LR-GMM where the parameters are defined as follows: enumerate all possible subsets of $[d]$ of cardinality $s$ and denote this set by $\mathcal{S} := \{S \subset [d]: |S|=s\}$. Note that there are $K = \binom{d}{s}$ possible subsets. Then $p_s$ is a LR-GMM of the form
\begin{align*}
    p_s = \sum_{S \in \mathcal{S}}  \frac{1}{\binom{d}{s}}\mathcal{N}(0, U_SU_S^T)
\end{align*} where $U_S  := [e_{i_1},\dots,e_{i_s}] \in \R^{d \times s}$ and $S := \{i_1,\dots,i_s\} \subset [d]$. Here, $e_i$ is the $i$-th standard basis vector in $\R^d$. In this case, note that our model set is given by $\Sigma = \bigcup_{S \in \mathcal{S}}\mathrm{Im}(U_S) = \{x \in \R^d : \|x\|_0 \leq s\}.$ Hence for $s$-sparse data $\hat{x}$, the previous convergence guarantees apply.
\end{remark}

\section{Experiments} \label{sec:exps}

We consider numerical experiments to support our theoretical findings. In particular, we will consider synthetic data from a LR-GMM and illustrate the convergence of the algorithm we analyze.

Given a LR-GMM distribution $p = \frac{1}{K}\sum_{k=1}^K\pi_K \mathcal{N}(0,U_kU_k^T)$, we aim to recover signals from a random Gaussian compressed sensing matrix $y = A\hat{x}$ where $A \in \R^{m \times d}$ with $A_{ij} \sim \mathcal{N}(0,1)$ and $\hat{x} \sim p$. Each $U_k \in \R^{d \times r}$ is a random matrix with orthonormal columns. We run our algorithm with various noise schedules $(\sigma_n)$ and a fixed step size of $\mu = 1.9 / \|A\|_2^2$. For our different choices of schedules, setting $N$ equal to the maximum number of iterations of our algorithm, we consider the following set of noise schedules: \begin{align*}
& \text{(geometric)}\  \sigma_n = \sigma_{\max}\left(\frac{\sigma_{\min}}{\sigma_{\max}}\right)^{n/N}, \\
& \text{(linear)}\  \sigma_n^2 = \sigma_{\max}^2 + \frac{n}{N}(\sigma_{\min}^2 - \sigma_{\max}^2), \\
& \text{(cosine)}\ \sigma_n = \sigma_{\min} +  \frac{1}{2}(\sigma_{\max} - \sigma_{\min})(1+\cos(\pi n / N)), \\
& \text{(infinite geometric)}\ \sigma_n = \sigma_{\max}  a^n\ \text{where}\ a \in (0,1).
\end{align*}

\begin{figure}[!h]
    \centering
    \includegraphics[width=0.45\linewidth]{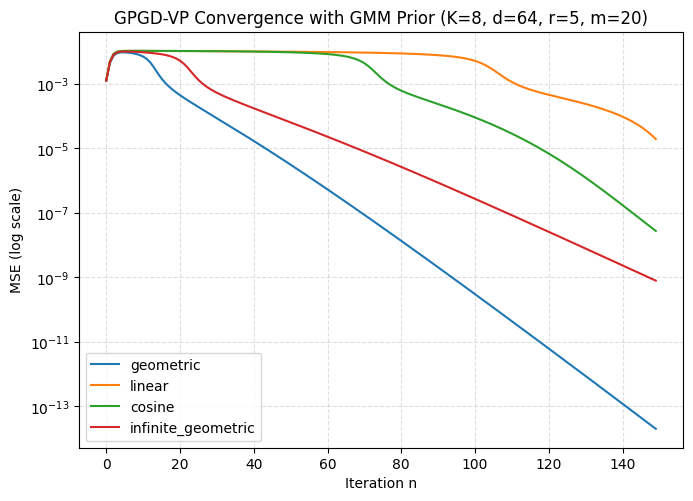}
    \caption{We take $\hat{x} \sim p = \frac{1}{8}\sum_{k=1}^8 \mathcal{N}(0,U_kU_k^T)$ and recover $\hat{x}$ from $y = A\hat{x}$ where $A\in \R^{m \times d}$ has i.i.d. $\mathcal{N}(0,1)$ entries. Here, $d = 64, r=5,$ and $m = 20$. We fix the step size $\mu = 1.9 / \|A\|_2^2$ and consider 4 different noise schedules with $\sigma_{\max} = 0.5$ and $\sigma_{\min} = 10^{-4}$. We show the mean squared error over all iterations with $N = 150$.}
    \label{fig:gmm-prior-convergence}
\end{figure}

Recovery results for our algorithm in the case when $K = 8$, $d = 64$, $r = 5$, and $m = 20$ are shown in Figure \ref{fig:gmm-prior-convergence}. We note that for a geometrically decaying noise schedule, the iterates exhibit linear convergence after an initial \textit{burn-in} period, which is shorter for the geometric schedule as compared to each of the other noise schedules. In Figure \ref{fig:distance-to-subspaces}, we further analyze this by zooming into the first part of the optimization and also display the distance of the iterates to each individual subspace $\mathrm{dist}(x_n,\mathrm{Im}(U_i)):=\|(I-U_iU_i^T)x_n\|_2$. We see that during this burn-in period, the iterates begin to reach the frontier, as their distance to each of the subspaces becomes close. After the correct subspace is identified, linear convergence is observed. 

\begin{figure}[!h]
    \centering
    \includegraphics[width=0.45\linewidth]{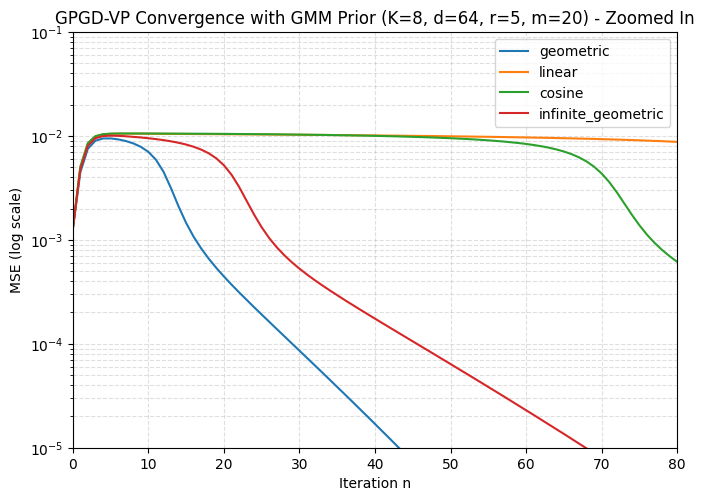}
    \includegraphics[width=0.5\linewidth]{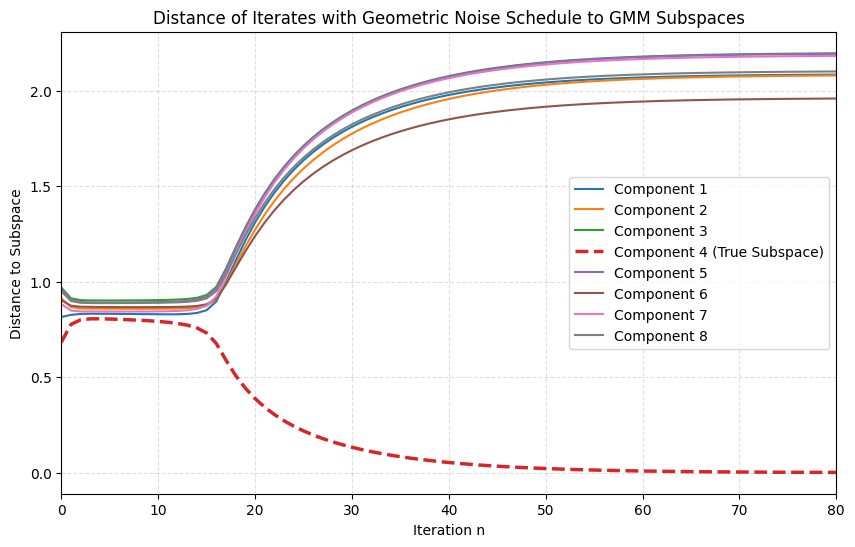}
    \caption{We consider the same experiment as in Figure \ref{fig:gmm-prior-convergence}. Top: a closer look at the first $80$ iterations of the algorithm. Bottom: for the geometric noise schedule, we calculate the distance of the iterates to each subspace, with the correct underlying subspace error in the dashed line. We see that during this burn-in period, the iterates move closer towards the frontier, but are propelled away from it after a short time.}
    \label{fig:distance-to-subspaces}
\end{figure}
\section{Discussion}

In this work, we developed a principled recovery theory for diffusion priors in linear inverse problems. Our analysis exploited the fact that noise-convolved scores can be interpreted as approximate projections onto low-dimensional model sets. This perspective allows us to interpret certain diffusion-based algorithms as generalized projected gradient descent methods with time-varying projections, enabling recovery guarantees across a range of measurement operators and data distributions. Under a restricted isometry condition, we showed that (i) for uniform measures on convex sets, the limiting denoiser converges to the metric projection with an explicit error bound, giving global convergence for recovering signals from this model set. Also, (ii) for low-rank Gaussian mixture models (LR-GMMs), we showed that the limiting operator recovers a projection operator onto a union-of-subspaces and we established global convergence to the true signal. These results naturally cover (iii) recovery of random sparse signals as well. In all cases, we demonstrated that the convergence rate depends explicitly on the choice of noise schedule. These results bridge modern generative priors with the well-studied theory of regularization in inverse problems, leaving several open directions for future work.

First, while our analysis treated specific data distributions, it would be natural to extend the theory to other rich classes of distributions, such as those supported on nonlinear manifolds. Second, integrating our framework into sampling approaches such as DPS \cite{Chungetal23} or diffusion-based plug-and-play methods \cite{feng2023score, xu2024provably} could yield novel sample complexity guarantees for posterior sampling algorithms. Third, many applications involve nonlinear forward operators (e.g., phase retrieval), and extending our analysis to such settings would broaden the impact of this theory. Finally, it would be important to develop robustness guarantees with respect to measurement noise or model mismatch between the learned diffusion prior and the true data distribution.

\bibliographystyle{abbrv}
\bibliography{thebib}
\newpage
\onecolumn


\section{Proofs for Section \ref{sec:general-conv-results}}

We bound the recovery error $  \|x_{n} -\hat{x}\|_2 $ with the following Lemma.
\begin{lem} \label{lem:bound_iterates_gen2}
Suppose $P_\Sigma $  is restricted $\beta$-Lipschitz with respect to $\Sigma$. Let $\delta : =\delta(\mu A^TA)$. Consider $x_n$ resulting from  iterations \eqref{eq:main-iterations}.
We have, for $n\geq 1$,
\begin{equation}\label{eq:final_bound_gen2}
\begin{split}
    \|x_{n} -\hat{x}\|_2 
     &\leq (\delta\beta)^{n} \| x_0 -\hat{x}  \|_2 + \|I-\mu A^TA\|_{\mathrm{op}} \sum_{l=0}^{n-1}(\delta \beta)^{n-l-1}  \|P^l(x_l) -P_\Sigma(x_l)\|_2. \\
\end{split}
\end{equation}
\end{lem}

\begin{proof}
 With the definition of the RIC $\delta:= \delta(\mu A^TA)$ and the restricted Lipschitz constant $\beta$ we have
\begin{equation} \label{eq:first_bound_gen2} 
\begin{split}
    \|x_{n+1} -\hat{x}\|_2 &=  \| (I-\mu A^TA) (P^n(x_n) -\hat{x}) \|_2 \\
    &=\| (I-\mu A^TA) (P_\Sigma(x_n)-\hat{x})  + (I-\mu A^TA) (P^n(x_n) -P_\Sigma(x_n)) \|_2 \\
     &\leq \delta \beta  \| x_n -\hat{x}  \|_2 + \|I-\mu A^TA\|_{\mathrm{op}} \|P^n(x_n) -P_\Sigma(x_n)) \|_2\\
\end{split}
\end{equation}

We verify the result by induction. The bound \eqref{eq:first_bound_gen2} for $n=0$ is exactly the case $n=1$ of \eqref{eq:final_bound_gen2}. Suppose \eqref{eq:final_bound_gen2} is verified at step $n$, then  

\begin{equation}
\begin{split}
    \|x_{n+1} -\hat{x}\|_2 \leq& \delta \beta  \left( (\delta \beta)^{n}  \| x_0 -\hat{x}  \|_2 + \|I-\mu A^TA\|_{\mathrm{op}} \sum_{l=0}^{n-1} (\delta \beta)^{n-l-1}  \|P^l(x_l) -P_\Sigma(x_l)\|_2\right)  \\
     &+ \|(I-\mu A^TA)\|_{\mathrm{op}}\|P^n(x_n) -P_\Sigma(x_n)\|_2\\
       =&  (\delta \beta)^{n+1}  \| x_0 -\hat{x}  \|_2 + \|I-\mu A^TA\|_{\mathrm{op}} \sum_{l=0}^{n-1} (\delta \beta)^{n-l} \|P^l(x_l) -P_\Sigma(x_l)\|_2\ \\
     &+ \|I-\mu A^TA\|_{\mathrm{op}} \|P^l(x_l) -P_\Sigma(x_l)\|_2\\
\end{split} \nonumber
\end{equation}
We remark that $\|P^n(x_n) -P_\Sigma(x_n)\|_2= (\delta \beta)^{n-n} \|P^n(x_n) -P_\Sigma(x_n)\|_2 $, which allows to conclude 

\begin{equation}
\begin{split}
    \|x_{n+1} -\hat{x}\|_2 &\leq (\delta \beta)^{n+1}  \| x_0 -\hat{x}  \|_2 + \|I-\mu A^TA\|_{\mathrm{op}} \sum_{l=0}^{n} (\delta \beta)^{n-l}  \|P^l(x_l) -P_\Sigma(x_l)\|_2. \\
\end{split} \nonumber
\end{equation}
\end{proof}

Lemma~\ref{lem:bound_iterates_gen2} bounds the recovery error by a geometric term and a term depending on the $\|P^{l}(x_l) -P_\Sigma(x_l) \|_2$. When these ``projection errors" are small enough we are able to show convergence.

\begin{proof}[Proof of Theorem~\ref{th:conv_var_proj}]
First remark that the sequence $\|x_n- x_0\|_2$ is bounded, i.e. there exists $D\geq 0$ such that $\|x_n- x_0\|_2 \leq D$. Indeed, with the hypothesis, there exists $C'\geq 0$ such that  $u_n =\|P^n(x_n)-P_\Sigma(x_n)\| \leq C'$, and from Lemma~\ref{lem:bound_iterates_gen2},
\begin{equation}
\begin{split}
    \|x_{n} -\hat{x}\|_2 
     &\leq \| x_0 -\hat{x}  \|_2 + \|I-\mu A^TA\|_{\mathrm{op}} C'\sum_{l=0}^{n-1}(\delta \beta)^{n-l-1}  . \\
     &\leq \| x_0 -\hat{x}  \|_2 +  \|I-\mu A^TA\|_{\mathrm{op}}  C'\sum_{l=0}^{\infty}(\delta \beta)^{l} =:D.
\end{split} \nonumber
\end{equation}

We use Lemma~\ref{lem:bound_iterates_gen2} with initialization $x_{n_0}$ for some $n_0 \geq 0$. We have
\begin{equation}
\begin{split}
    \|x_{n+n_0} -\hat{x}\|_2 
     &\leq ( \delta \beta)^{n}  \| x_{n_0} -\hat{x}  \|_2 + \|I-\mu A^TA\|_{\mathrm{op}} \sum_{l=0}^{n-1} \left(\beta\delta \right)^{n-1-l}  \|P^{l+n_0}(x_{l+n_0}) -P_\Sigma \|_2 \\
      &\leq D( \delta \beta)^{n}   + \|I-\mu A^TA\|_{\mathrm{op}} \sum_{l=0}^{n-1} \left(\beta\delta \right)^{n-1-l} \max_{l=0,n-1}\|P^{n_0+l}(x_{n_0+l}) -P_\Sigma(x_{n_0+l})  \|_2 \\
      &\leq  D( \delta \beta)^{n}   + \|I-\mu A^TA\|_{\mathrm{op}} \frac{1}{1 - \beta\delta} \max_{l=0,n-1}\|P^{n_0+l}(x_{n_0+l}) -P_\Sigma(x_{n_0+l})  \|_2 \\
\end{split} \nonumber
\end{equation}
Taking $n=n_0$, we have 

\begin{equation}
\begin{split}
    \|x_{2n} -\hat{x}\|_2 
      &\leq D ( \delta \beta)^{n}  + \frac{\|I-\mu A^TA\|_{\mathrm{op}} }{1 - \beta\delta}  \max_{l=0,n-1}\|P^{n+l}(x_{n+l}) -P_\Sigma(x_{n+l})  \|_2  \\
\end{split} \nonumber
\end{equation}
Setting $C = \max  (D,  \frac{\|I-\mu A^TA\|_{\mathrm{op}} }{1 - \beta\delta}  )$, we have 
\begin{equation}
\begin{split}
    \|x_{2n} -\hat{x}\|_2 
      &\leq  C( ( \delta \beta)^{n}   + \max_{l=0,n-1}\|P^{n+l}(x_{n+l}) -P_\Sigma(x_{n+l})  \|_2 )\\
       &\leq  C (( \delta \beta)^{n}   +  \max_{l=0,n}\|P^{n+l}(x_{n+l}) -P_\Sigma(x_{n+l})  \|_2 ).\\ \nonumber
\end{split}
\end{equation}
Taking $n=n_0-1$, we have 

\begin{equation}
\begin{split}
    \|x_{2n+1} -\hat{x}\|_2 
      &\leq  C(( \delta \beta)^{n}   +   \max_{l=0,n-1}\|P^{n+1+l}(x_{n+1+l}) -P_\Sigma(x_{n+1+l}) \|_2 )\\
      &\leq  C( ( \delta \beta)^{n}   + \max_{l=1,n}\|P^{n+l}(x_{n+l}) -P_\Sigma(x_{n+l}) \|_2 )\\
      &\leq  C ( (\delta \beta)^{n}   + \max_{l=0,n}\|P^{n+l}(x_{n+l}) -P_\Sigma(x_{n+l}) \|_2 ).\\ \nonumber
\end{split}
\end{equation}
Hence for any  $m =  2n +1$ or $m=2n$, we have $n=\lfloor m/2\rfloor$, we deduce 
\begin{equation}
\begin{split}
     \|x_{m} -\hat{x}\|_2 
             &\leq  C( ( \delta \beta)^{\lfloor m/2\rfloor}   + \max_{l=\lfloor m/2\rfloor,\lfloor m/2\rfloor+\lfloor m/2\rfloor}\|P^{l}(x_l) -P_\Sigma(x_l) \|_2)\\
              &\leq  C( ( \sqrt{\delta \beta})^{ m}   + \max_{l=\lfloor m/2\rfloor,m}\|P^{l}(x_l) -P_\Sigma(x_l) \|_2 ).\\ \nonumber
\end{split}
\end{equation}
\end{proof}

\begin{proof}[Proof of Corollary~\ref{cor:slow_convergence}]
Let $u_n := \|P^n(x_n) -P_\Sigma(x_n)\|_2 $. By  hypothesis, $(u_n)$ is a positive sequence that converges to $0$.  Hence $\lim_{n\to \infty} \sup_{l\geq n} u_l = \lim_{n \to \infty} u_n =0$ and with Theorem~\ref{th:conv_var_proj}, 
\begin{align*}
\|x_n-\hat{x}\| & \leq C ((\sqrt{\delta \beta})^n +\max_{l=\lfloor n/2\rfloor,n}\|P^{l}(x_l) -P_\Sigma(x_l) \|_2 ) \\
& \leq  C( (\sqrt{\delta \beta})^n +\sup_{l \geq \lfloor n/2\rfloor}u_l ) \xrightarrow[n \to \infty]{}  0 .
\end{align*}

\end{proof}

\begin{proof}[Proof of Corollary~\ref{cor:linear_convergence}]

From the proof of Lemma~\ref{lem:bound_iterates_gen2}, we have
\begin{equation} 
\begin{split}
    \|x_{n+1} -\hat{x}\|_2 
     &\leq \delta \beta  \| x_n -\hat{x}  \|_2 + \|I-\mu A^TA\|_{\mathrm{op}} \|P^n(x_n) -P_\Sigma(x_n) \|_2\\
     & \leq \delta \beta  \| x_n -\hat{x}  \|_2+\|I-\mu A^TA\|_{\mathrm{op}} \|P^n -P_\Sigma \|_{\mathrm{op}} \|x_n\|_2 \\ \nonumber
     \end{split}
\end{equation}
and, by hypothesis, there exists $C'$ such that
\begin{equation} 
\begin{split}
    \|x_{n+1} -\hat{x}\|_2 
     & \leq \delta \beta  \| x_n -\hat{x}  \|_2+ C'e^{-cn} \|x_n\|_2\\
     &\leq (\delta \beta +C'e^{-cn})  \| x_n -\hat{x}  \|_2+ C'e^{-cn} \|\hat{x}\|_2.\\ \nonumber
     \end{split}
\end{equation}
We deduce that there is $n_1$, $\rho < 1$, such that for $n \geq n_1$
\begin{equation} 
\begin{split}
    \|x_{n+1} -\hat{x}\|_2  &\leq \rho \| x_n -\hat{x}  \|_2+ C'e^{-cn} \|\hat{x}\|_2.\\ \nonumber
     \end{split}
\end{equation}
By induction, 
\begin{equation} 
\begin{split}
    \|x_{n_1+n} -\hat{x}\|_2  &\leq \rho^n \| x_{n_1} -\hat{x}  \|_2+  \sum_{i=0}^{n-1}\rho^{n-i-1} C'e^{-c(n+n_1)} \|\hat{x}\|_2.\\
    & \leq \rho^n \| x_{n_1} -\hat{x}  \|_2+  \sum_{i=0}^{n-1}\rho^{n-i-1}C'\|\hat{x}\|_2.\\ 
    & \leq \rho^n \| x_{n_1} -\hat{x}  \|_2+  \frac{1}{1-\rho}C'\|\hat{x}\|_2.\\ \nonumber
     \end{split}
\end{equation}
We deduce that $(x_n) $ is bounded and that $\lim_{n\to\infty}\|P_n(x_n)-P_\Sigma(x_n)\|_2 =0$.

We conclude by remarking that there exists $D \geq 0$ such that
\begin{align*}
\max_{l=\lfloor n/2\rfloor,n}\|P^{l}(x_l) -P_\Sigma(x_l) \|_{\mathrm{op}} &  \leq D \max_{l=\lfloor n/2\rfloor,n} e^{-cl} \leq D e^{-\frac{c}{2}n}.
\end{align*}
and by applying Theorem~\ref{th:conv_var_proj}.
\end{proof}

\section{{Proofs for Section \ref{sec:convex-sets}}} \label{sec:unif-measure-convex-set}

To prove Proposition \ref{prop:uniform-convex}, we first require some additional notation and topological preliminaries. For a convex, compact set $\Sigma \subset \R^d$, let $\mathrm{aff}(\Sigma )$ denote the affine hull of $\Sigma $, which corresponds to the smallest affine set that contains $\Sigma $: $$\mathrm{aff}(\Sigma) := \left\{\sum_{i=1}^k \lambda_i x_i: k > 0,\ x_i \in \Sigma,\ \lambda_i \in \R,\ \sum_{i=1}^k \lambda_i = 1\right\}.$$ The intrinsic dimension of $\Sigma$ is then given by $s = \mathrm{dim}(\mathrm{aff}(\Sigma))$. Note that when $\Sigma$ is a full-dimensional convex body (compact, convex with non-empty interior), the affine hull coincides with the entire ambient space, i.e., $\mathrm{aff}(\Sigma) = \R^d$. Within $\mathrm{aff}(\Sigma)$, we will consider several sets, such as the relative intrinsic ball in $\mathrm{aff}(\Sigma)$: \begin{align*}
    B_{\mathrm{aff}(\Sigma)}(x,r) := \{z \in \mathrm{aff}(\Sigma): \|z-x\|_2 < r\} = \mathrm{aff}(\Sigma) \cap B(x,r).
\end{align*} We equip $\mathrm{aff}(\Sigma)$ with the subspace topology, meaning that open sets in $\mathrm{aff}(\Sigma)$ are of the form $\mathrm{aff}(\Sigma) \cap U$ with $U$ open in $\R^d$. Then the relative interior of $\Sigma$, $\mathrm{ri}(\Sigma)$, is $$\mathrm{ri}(\Sigma) := \{x \in \Sigma : \exists \epsilon > 0\ \text{s.t.}\ B_{\mathrm{aff}(\Sigma)}(x,\epsilon) \subseteq \Sigma\}.$$ Then for compact $E$, we define the relative boundary of $E$ with respect to $\mathrm{aff}(\Sigma)$ to be $\partial_{\mathrm{aff}(\Sigma)}\Sigma := \Sigma \setminus \mathrm{ri}(\Sigma).$ For our uniform measure, let $\lambda_s$ denote the $s$-dimensional Hausdorff measure restricted to $\mathrm{aff}(\Sigma)$. Then the uniform probability measure on $E$ (relative to $\mathrm{aff}(\Sigma)$) is $$dp(x) = \frac{\mathbf{1}_{\{x \in \Sigma\}}}{\lambda_s(\Sigma)}d\lambda_s(x),\ x \in \mathrm{aff}(\Sigma).$$ We are now ready to state the main result.
\begin{claim}[Proposition \ref{prop:uniform-convex} in the main body] \label{claim:convex-projection}
    Suppose $\Sigma \subset \R^d$ is a compact, convex set containing the origin with non-empty relative interior in $\mathrm{aff}(\Sigma)$. Set $s = \mathrm{dim}(\mathrm{aff}(\Sigma))$ and let $p$ be the uniform measure on $E$. Then we have that there exists a positive absolute constant $C$ such that the following holds for any $y\in \R^d$: \begin{align*}
        \|D_{\sigma}(y) - P^{\perp}_{\Sigma}(y)\|_2^2 \leq C \left(\frac{s R_{\Sigma}^2 \max\{1,\mathrm{dist}(y,E)\}^s\lambda_s(\Sigma)}{\kappa_{\Sigma,s}} \right)\cdot \sigma^2 \log(1/\sigma)
    \end{align*} where $R_{\Sigma}:= \sup_{z \in \Sigma}\|z\|$ and $\kappa_{\Sigma,s}$ is the constant defined in Lemma \ref{lem:technical-volume-est}. For a noise schedule $\sigma_n$, defining $P^n(x):= D_{\sigma_n}(x)$ as in \eqref{eq:appx-projection-op}, then we have for $n$ such that $\sigma_n < 1$, $$\|P^n - P_{\Sigma}^\perp\|_{\mathrm{op}} := \sup_{\|x\|_2=1}\|P^n(x) - P_{\Sigma}^\perp(x)\|_2 \leq C\sqrt{\left(\frac{s R_{\Sigma}^2\lambda_s(\Sigma)}{\kappa_{\Sigma,s}} \right)\cdot \sigma_n^2 \log(1/\sigma_n)}.$$
\end{claim}
\begin{proof}
\textbf{Step 1: Connection between denoiser and truncated Gaussian:} We first show that the denoiser $D_{\sigma}(y)$ can be written as $P_{\Sigma}^\perp(y)=x^*=x^*(y)$ plus an additional term that is the conditional mean of a truncated Gaussian distribution. First, note that the posterior of $x \sim p$ given $y = x + \sigma z$ with $z \sim \mathcal{N}(0,I)$ is proportional to $\exp(-\|y-x\|^2_2/(2\sigma^2)\mathbf{1}_{\{x \in \Sigma\}}$ which gives \begin{align*}
    D_{\sigma}(y) = \mathbb{E}[x|y]= \frac{\int_\Sigma x \exp(-\|y-x\|_2^2/2\sigma^2)d\lambda_s(x)}{\int_{\Sigma}\exp(-\|y-u\|_2^2/2\sigma^2)d\lambda_s(u)}.
\end{align*} 

Using the substitutions $z = x-x^*$ and  $z = u-x^*$ with $\Sigma_* := \Sigma-x^*$, we have that 
\begin{align*}
    D_{\sigma}(y) = x^*+ \frac{\int_{\Sigma_*} z \exp(-\|y-z-x^*\|_2^2/2\sigma^2)d\lambda_s(z)}{\int_{\Sigma_*} \exp(-\|y-z-x^*)\|_2^2/2\sigma^2)d\lambda_s(z)}.
\end{align*}
Thus, we get the following representation of the denoiser:
\begin{align*}
     D_{\sigma}(y) - x^* 
     = \frac{\int_{\Sigma_*}z \exp\left(-\frac{\|z - (y-x^*)\|_2^2}{2\sigma^2}\right)d\lambda_s(z)}{\int_{\Sigma_*} \exp\left(-\frac{\|z - (y-x^*)\|_2^2}{2\sigma^2}\right)d\lambda_s(z)}.
 \end{align*}
 \color{black}

Hence this quantity has the interpretation as the mean of a degenerate, truncated Gaussian $\mathcal{N}(y-x^*,\sigma^2 I; \Sigma_*)$ which has measure $dp_{\sigma}(z):= \mathbf{1}_{\{z \in \Sigma_*\}} \exp(-\|z-(y-x^*)\|_2^2/2\sigma^2)d\lambda_s(z) / Z_{\sigma}$ with $Z_{\sigma} = \int_{\Sigma_*}\exp(-\|z-(y-x^*)\|_2^2/2\sigma^2)d\lambda_s(z)$, i.e., $D_{\sigma}(y) - x^*= \mathbb{E}_{z \sim \mathcal{N}(y-x^*,\sigma^2 I; \Sigma_*)}[z]$. The next step decomposes an upper bound of $ \|D_{\sigma}(y) - x^*\|_2^2$ depending on an arbitrary parameter $\delta > 0$.

\textbf{Step 2: General upper bound for $\delta > 0$:} Note that by Jensen's inequality, it suffices to bound \begin{align} \label{eq:bound_res1}
    \|D_{\sigma}(y) - x^*\|_2^2= \left\|\mathbb{E}_{z \sim \mathcal{N}(y-x^*,\sigma^2 I; \Sigma_*)}[z]\right\|_2^2 \leq \mathbb{E}_{z \sim \mathcal{N}(y-x^*,\sigma^2 I; \Sigma_*)}[\|z\|_2^2].
\end{align} Fix any $\delta \in (0,1)$ and consider the events $\mathbf{1}_{\|z\|_2>\delta}$ and $\mathbf{1}_{\|z\|_2 \leq \delta}$. Set $R:= \sup_{z \in \Sigma_*}\|z\|_2 < \infty$, which is finite since $\Sigma$ is compact. Then we note that \begin{align*}
    \mathbb{E}_{z \sim \mathcal{N}(y-x^*,\sigma^2 I; \Sigma_*)}[\|z\|_2^2]  & = \mathbb{E}_{z \sim \mathcal{N}(y-x^*,\sigma^2 I; \Sigma_*)}[\|z\|_2^2\mathbf{1}_{\|z\|_2>\delta}] + \mathbb{E}_{z \sim \mathcal{N}(y-x^*,\sigma^2 I; \Sigma_*)}[\|z\|_2^2\mathbf{1}_{\|z\|_2\leq\delta}] \\
    & \leq R^2 \mathbb{P}_{z \sim \mathcal{N}(y-x^*,\sigma^2 I; \Sigma_*)}(\|z\|_2 > \delta) + \delta^2.
\end{align*} We will estimate $\mathbb{P}(\|z\|_2>\delta).$ Observe that \begin{align*}
    \mathbb{P}_{z \sim \mathcal{N}(y-x^*,\sigma^2 I; \Sigma_*)}(\|z\|_2>\delta) & = \frac{\int_{z\in \Sigma_{*},\|z\|_2>\delta } \exp\left(-\frac{\|z - (y-x^*)\|_2^2}{2\sigma^2}\right)d\lambda_s(z)}{\int_{\Sigma_*} \exp\left(-\frac{\|z - (y-x^*)\|_2^2}{2\sigma^2}\right)d\lambda_s(z)}.
\end{align*} First, note that for $z \in \Sigma_*$, we have by convexity of $\Sigma_*$ that $\langle z, y-x^*\rangle \leq 0$. This implies $$\|z-(y-x^*)\|_2^2=\|z\|_2^2 + \|y-x^*\|_2^2 - 2\langle z,y-x^*\rangle \geq \|z\|_2^2 + \|y-x^*\|_2^2.$$ Hence we have the following bound on the numerator: 
\begin{align*}
    \int_{z\in \Sigma_{*},\|z\|_2>\delta  } \exp\left(-\frac{\|z - (y-x^*)\|_2^2}{2\sigma^2}\right)d\lambda_s(z) & \leq \int_{z\in \Sigma_{*},\|z\|_2>\delta  }\exp\left(- \frac{\delta^2 + \|y-x^*\|_2^2}{2\sigma^2}\right)d\lambda_s(z) \\
    &\leq \lambda_s(\Sigma_*)\exp\left(-\frac{\delta^2+\|y-x^*\|_2^2}{2\sigma^2}\right).
\end{align*} For the denominator, note that $0 \in \Sigma_*$ since $x^* \in \Sigma$ so $x^*-x^* =0\in \Sigma_*$. Consider a sufficiently small parameter $\varepsilon$ such that $\varepsilon < \frac{\delta^2}{4\max\{1,\|y-x^*\|_2\}}$. 
For $z \in B_{\mathrm{aff}(\Sigma)}(0,\varepsilon) \cap \Sigma_*$, we have that $\|z-(y-x^*)\| \leq \varepsilon+\|y-x^*\|$. Then we have that \begin{align*}
    \int_{\Sigma_*} \exp\left(-\frac{\|z-(y-x^*)\|_2^2}{2\sigma^2}\right)dz & \geq \int_{B_{\mathrm{aff}(\Sigma)}(0,\varepsilon)\cap \Sigma_*} \exp\left(-\frac{\varepsilon^2+2\varepsilon\|y-x^*\|_2+\|y-x^*\|_2^2}{2\sigma^2}\right)d\lambda_s(z) \\
    & = \exp\left(-\frac{\varepsilon^2+2\varepsilon\|y-x^*\|_2+\|y-x^*\|_2^2}{2\sigma^2}\right) \lambda_s(B_{\mathrm{aff}(\Sigma)}(0,\varepsilon)\cap \Sigma_*).
\end{align*} Combining both bounds, we have that \begin{align*}
    \mathbb{P}_{z \sim \mathcal{N}(y-x^*,\sigma^2 I; \Sigma_*)}(\|z\|_2>\delta) & \leq \frac{\lambda_s(\Sigma_*)}{\lambda_s(B_{\mathrm{aff}(\Sigma)}(0,\varepsilon) \cap \Sigma_*)}\exp\left(-\frac{\delta^2+\|y-x^*\|_2^2}{2\sigma^2}+\frac{\varepsilon^2+2\varepsilon\|y-x^*\|+\|y-x^*\|_2^2}{2\sigma^2}\right) \\
    & = \frac{\lambda_s(\Sigma_*)}{\lambda_s(B_{\mathrm{aff}(\Sigma)}(0,\varepsilon) \cap \Sigma_*)}\exp\left(\frac{-\delta^2+\varepsilon^2+2\varepsilon\|y-x^*\|_2}{2\sigma^2}\right) \\
    & \leq \frac{\lambda_s(\Sigma_*)}{\lambda_s(B_{\mathrm{aff}(\Sigma)}(0,\varepsilon) \cap \Sigma_*)}\exp\left(\frac{-\frac{\delta^2}{2}+\varepsilon^2}{2\sigma^2}\right)
\end{align*} where we used the fact that $\varepsilon < \delta^2/(4\|y-x^*\|_2)$ in the last inequality. 
By setting $\varepsilon := \frac{\delta^2}{8\max\{1,\|y-x^*\|_2\}}$, we have that $-\delta^2/2+\varepsilon^2 < 0$. This implies that we have positive constants $C(\delta,y,x^*)$ and $c(\delta,y,x^*)$ that do not depend on $\sigma$ such that 
\begin{align*}
    \mathbb{P}_{z \sim \mathcal{N}(y-x^*,\sigma^2 I; \Sigma_*)}(\|z\|_2>\delta) \leq C(\delta,y,x^*) \exp\left(- \frac{c(\delta,y,x^*)}{2\sigma^2}\right).
\end{align*} Hence, for any fixed $\delta > 0$ we have that 
\begin{align} \label{eq:bound_exp1}
    \mathbb{E}_{z \sim \mathcal{N}(y-x^*,\sigma^2 I; \Sigma_*)}[\|z\|_2^2] \leq \delta^2 + R^2 C(\delta,y,x^*) \exp\left(- \frac{c(\delta,y,x^*)}{2\sigma^2}\right) 
\end{align} where $R := \sup_{z \in \Sigma_*}\|z\|_2$, $c(\delta) = \delta^2/2 - \varepsilon$ and $C(\delta):= \frac{\lambda_s(\Sigma_*)}{\lambda_s(B_{\mathrm{aff}(\Sigma)}(0,\varepsilon) \cap \Sigma_*)}$.

\textbf{Step 3: Simplification for specific choice of $\delta(\sigma)$:} Since the above bound holds for any $\delta > 0$, we will choose a specific value to minimize the above quantity. First, note that since $x^* \in \Sigma$, then we have $\lambda_s(B_{\mathrm{aff}(\Sigma)}(0,\varepsilon)\cap\Sigma_*)=\lambda_s(B_{\mathrm{aff}(\Sigma)}(x^*,\varepsilon)\cap \Sigma))$. Note that by Lemma \ref{lem:technical-volume-est}, there exists an explicit $\kappa$, defined in \eqref{eq:kappa-def}, such that 
$$C(\delta,y,x^*) \leq \frac{\lambda_s(\Sigma_*)}{\kappa \varepsilon^s} = \frac{\lambda_s(\Sigma_*)8^s\max\{1,\|y-x^*\|_2\}^s}{\kappa \delta^{2s}} =: C_1\delta^{-2s}.$$ 
We also have that 
\begin{align} \label{eq:def_C1_convex_case}
        c(\delta,y,x^*) = \frac{\delta^2}{2} -\frac{\delta^2}{8\max\{1,\|y-x^*\|_2\}} \geq \frac{\delta^2}{4}
    \end{align} 
    so that, with~\eqref{eq:bound_res1} and \eqref{eq:bound_exp1}, we get the bound 
    \begin{align*}
        \|D_{\sigma}(y) - x^*\|_2^2 \leq \delta^2 + R^2C_1 \delta^{-2s} \exp\left(- \frac{\delta^2}{8\sigma^2}\right), \forall \delta > 0.
    \end{align*} Observe that the term on the right can be simplified as \begin{align*}
        (\delta^{2})^{-s}\exp\left(-\frac{\delta^2}{8\sigma^2}\right) &= \exp\left(-s\log(\delta^2) - \frac{\delta^2}{8\sigma^2}\right)
    \end{align*} Note that if we choose $$\delta^2 = 8 (\alpha + 2s)\sigma^2 \log (1/\sigma),\ \alpha > 0$$ we obtain \begin{align*}
        \exp\left(-s\log(\delta^2) - \frac{\delta^2}{8\sigma^2}\right) & = \exp\left[-s\left(\log 8 + \log (\alpha + 2s) + \log (\sigma^2) + \log\log(1/\sigma)\right)\right]\exp\left[-(\alpha + 2s)\log(1/\sigma)\right]\\
        & = 8^{-s}(\alpha+2s)^{-s}(\sigma^2)^{-s}[\log(1/\sigma)]^{-s}\cdot\sigma^{\alpha + 2s}\\
        & = (8(\alpha + 2s))^{-s}\sigma^{\alpha}[\log(1/\sigma)]^{-s}.
    \end{align*}

This gives the bound \begin{align*}
         \|D_{\sigma}(y) - x^*\|_2^2 & \leq 8 (\alpha + 2s)\sigma^2 \log (1/\sigma) +  R^2 C_1 (8(\alpha + 2s))^{-s}\sigma^{\alpha}[\log(1/\sigma)]^{-s}.
    \end{align*} Setting $\alpha = 2$ and using the definition of $C_1$ from~\eqref{eq:def_C1_convex_case} gives the following upper bound: \begin{align*}
        \|D_{\sigma}(y) - x^*(y)\|_2^2 &  \leq 8(2 + 2s)\left(\frac{s R^2 \lambda_s(\Sigma)\max\{1,\mathrm{dist}(y,\Sigma)\}^s}{\kappa}\right)\cdot\sigma^2 \log (1/\sigma).
    \end{align*} The final operator bound in the result uses the fact that for $\|x\|_2=1,$ we have that $\mathrm{dist}(x,\Sigma)=\min_{z \in \Sigma}\|x-z\|_2 \leq \|x-0\|_2=\|x\|_2=1$ since $0 \in \Sigma.$
\end{proof}

    

\begin{lem} \label{lem:technical-volume-est}
    Let $\Sigma \subset \R^d$ be a compact, convex set containing the origin with non-empty relative interior in $\mathrm{aff}(\Sigma)$ and set $s = \mathrm{dim}(\mathrm{aff}(\Sigma))$. Then there exists a constant $\kappa > 0$ that depends on $\Sigma$ such that for any $\varepsilon \in (0,1)$, \begin{align*}
        \lambda_s(B_{\mathrm{aff}(\Sigma)}(x, \varepsilon) \cap E) \geq \kappa \epsilon^s\ \forall x \in \partial_{\mathrm{aff}(\Sigma)} \Sigma.
    \end{align*} A specific value for $\kappa$ can be taken as follows: if $\kappa_s$ denotes the volume of the unit Euclidean ball in $\R^s$ and we set $r := \sup\{t >0 : B_{\mathrm{aff}(\Sigma)}(0,t) \subseteq E\}$ and $R := \inf\{t > 0 : \Sigma \subseteq B_{\mathrm{aff}(\Sigma)}(0,t)\}$, then \begin{align}
        \kappa := \kappa_s \cdot \left(\frac{\min\{r,1\}}{2\max\{R+r,1\}}\right)^s. \label{eq:kappa-def}
    \end{align}
\end{lem}
\begin{proof}
    Fix $\varepsilon \in (0,1)$ and $x \in \partial_{\mathrm{aff}(\Sigma)} \Sigma$. Since $\Sigma$ is compact, convex with non-empty (relative) interior in $\mathrm{aff}(\Sigma)$, there exists $0 < r < R < \infty$ such that $B_{\mathrm{aff}(\Sigma)}(0,r) \subseteq \Sigma \subseteq B_{\mathrm{aff}(\Sigma)}(0,R)$. Now, consider $$\delta :=  \frac{\varepsilon}{2\max\{R + r,1\}}.$$ Note that by construction we have that $\delta \leq \varepsilon / 2 < 1/2$. Set $x_{\delta} := (1-\delta)x \in E$ and consider the ball $B_{\mathrm{aff}(\Sigma)}(x_{\delta}, \delta \min\{r,1\})$. We claim that $B_{\mathrm{aff}(\Sigma)}(x_{\delta},\delta \min\{r,1\}) \subseteq B_{\mathrm{aff}(\Sigma)}(x,\varepsilon) \cap \Sigma$. Indeed, let $y \in B_{\mathrm{aff}(\Sigma)}(x_{\delta},\delta \min\{r,1\})$. To see that $y \in E$, note that we can write $$y = (1-\delta)x + \delta z\ \text{where}\ z \in B_{\mathrm{aff}(\Sigma)}(0,\min\{r,1\}) \subseteq B_{\mathrm{aff}(\Sigma)}(0,r) \subseteq \Sigma.$$ Since $z,x \in \Sigma$ and $\Sigma$ is convex, we have that the linear combination $y = (1-\delta)x + \delta z \in \Sigma$. Moreover, we have that $y \in B_{\mathrm{aff}(\Sigma)}(x,\varepsilon)$ since $y \in \mathrm{aff}(\Sigma)$ and \begin{align*}
        \|y - x\|_2 & \leq \|y - x_{\delta}\|_2 + \|x_{\delta} - x\|_2 \\
        & \leq \delta \min\{r,1\} +\delta\|x\|_2 \\
        & \leq \delta(1+R) \\
        & \leq \frac{\varepsilon}{2}\cdot2 =\varepsilon.
    \end{align*} Hence we have that $B_{\mathrm{aff}(\Sigma)}(x_{\delta},\delta \min\{r,1\}) \subseteq B_{\mathrm{aff}(\Sigma)}(x,\varepsilon) \cap \Sigma$. Let $\kappa_s := \mathrm{vol}(\{z \in \R^s : \|z\|_2 < 1\})$ be the $s$-dimensional volume of the unit Euclidean ball in $\R^s$. Using our above result, we have that \begin{align*}
        \lambda_s(B_{\mathrm{aff}(\Sigma)}(x,\varepsilon) \cap \Sigma) & \geq \lambda_s(B_{\mathrm{aff}(\Sigma)}(x_{\delta},\delta \min\{r,1\})) \\
        & = \mathrm{vol}(\{z \in \R^s : \|z\|_2 \leq\delta \min\{r,1\}\}) \\
        & = \kappa_s \cdot \delta^s \cdot \min\{r,1\}^s \\
        & = \kappa_s \cdot \left(\frac{\min\{r,1\}}{2\max\{R+r,1\}}\right)^s\varepsilon^s \\
        & = : \kappa \varepsilon^s.
    \end{align*}
\end{proof}

We now give a short proof of the corresponding convergence guarantee when the data distribution is uniform on a convex set. 
\begin{proof}[Proof of Corollary \ref{cor:convex-convergence}]
By Proposition \ref{prop:uniform-convex}, we have that the iteration-dependent operator $P^n$ converges uniformly to the metric projection $P^{\perp}_{\Sigma}$ at the rate $\|P^n - P^{\perp}_{\Sigma}\|_{\mathrm{op}} \lesssim \sigma_n \sqrt{\log 1/\sigma_n}.$ Since $\beta = \beta_{\Sigma}(P^{\perp}_{\Sigma}) = 1$ and $\delta = \delta(\mu A^TA)< 1$, the assumptions of Theorem \ref{th:conv_var_proj} are satisfied and we can apply to the operator norm bound to \eqref{eq:general-iterates-bound-GPGD-VP} to get \begin{align*}
    \|x_n -\hat{x}\|_2 \leq C \left(\delta^{n/2} + \sigma_n \sqrt{\log (1/\sigma_n)}\right).
\end{align*}    
\end{proof}

\section{Proofs for Section \ref{sec:LR-GMM-theory}}
In the following, to make notations lighter, we use $t= t_n =\sigma_n^2$.

\subsection{Computation of limiting projection for LR-GMM}
For LR-GMM, we look at the limiting projection of 
\begin{equation} \label{eq:expr_P_n}
\begin{split}
    P^n(x)& \;= \frac{1}{1+t}
\frac{\sum_{k=1}^K 
       \pi_k\,\mathcal{N}\!\bigl(x;\,0,\;U_kU_k^T + tI\bigr)\,
       U_kU_k^T\,x}{
      \sum_{\ell=1}^K 
       \pi_\ell\,\mathcal{N}\!\bigl(x;\,0,\;U_\ell U_\ell^T + tI\bigr)} \\
       & =  \frac{1}{1+t} \sum_{k=1}^K \omega_k(x,t) P_{E_k}^\perp(x)
       \end{split} \nonumber
\end{equation}
where  $ P_{E_k}^\perp = U_kU_k^T $ and where we defined  
\begin{equation}
\begin{split}
   \omega_k(x,t)  & := \frac{
       \pi_k\,\mathcal{N}\!\bigl(x;\,0,\;U_kU_k^T + tI\bigr)}{
      \sum_{\ell=1}^K 
       \pi_\ell\,\mathcal{N}\!\bigl(x;\,0,\;U_\ell U_\ell^T + tI\bigr)}. \\
       \end{split} \nonumber
\end{equation}

Note that these quantities are not obviously defined at $t=0$  because writing $ \mathcal{N}(x;0,U_{\ell}U_{\ell}^T)$ makes the assumption that the probability measure associated with a low-dimensional Gaussian has a density in $\mathbb{R}^d$, but this is not the case. 

 However, the quantities $ \omega_k(x,t)$ have a limit in $t=0$. 
\begin{lem}\label{lem:limit_omega}
Consider a low-rank GMM supported on $K$ subspaces $(E_\ell)_{\ell=1}^K$. We have for all $x \in \mathbb{R}^d$
\begin{equation}
\lim_{t \to 0} \omega_k(x,t) =
\begin{cases}
1 & \text{if } \forall \ell \neq k, \|P_{E_k}^{\perp}(x)\|_2^2 > \|P_{E_\ell}^{\perp}(x)\|_2^2, \\
0 & \text{if } \exists \ell \neq k, \|P_{E_k}^{\perp}(x)\|_2^2 < \|P_{E_\ell}^{\perp}(x)\|_2^2.
\end{cases} \nonumber
\end{equation}

Now assume that $\dim (E_k) = r$ for all $k$. Given $x \in \mathbb{R}^d$, define  $L = \{ \ell :\|P_{E_\ell}^{\perp}(x)\|_2^2 = \|P_{\Sigma}^{\perp}(x)\|_2^2 \}$. Then
we have 
\begin{equation}
\lim_{t \to 0} \omega_k(x,t) =
\begin{cases}
\frac{\pi_k}{\sum_{\ell \in L} \pi_\ell} & \text{if } k \in L, \\
0 & \text{if } k \in L^c.
\end{cases} \nonumber
\end{equation}

where $L^c = \{1,\ldots,K\} \setminus L$.

We consequently define $\omega_k(x,0):= \lim_{t\to 0}\omega_k(x,t) $.
\end{lem}

\begin{proof}
We have, for $x \in \mathbb{R}^d$ and $t > 0$, 
\begin{equation}
\begin{split}
\omega_k(x,t)&:= \frac{\nu_k(x,t)}{\sum_{\ell=1}^K \nu_l(x,t)} \\
\end{split} \nonumber
\end{equation}
with 
\begin{equation}
\nu_k(x,t) =  \frac{\pi_k}{(2\pi)^{d/2}\det^{1/2}(U_kU_k^T + tI)} \cdot \exp\left(-\frac{1}{2}x^T(U_kU_k^T + tI)^{-1}
x\right) . \nonumber
\end{equation} 
Let $r_k = \dim (E_k)$. We have $\det(U_kU_k^T + tI) =(1+t)^{r_k}t^{d-r_k}.$

Let $x= u +u^\perp$ with $u  =P_{E_k}^\perp(x)$. Consider $V_k$ an orthonormal completion of $U_k$. We have 
\begin{equation}
 x^T(U_kU_k^T + t I)^{-1}x =  x^T (V_k D_tV_k^T)^{-1}x \nonumber
\end{equation}
where $D_t$ is a diagonal matrix where $D_t(i,i) = (1+t) $ for $i \leq r_k$ and $D_t(i,i) = t $ for $i> r_k$. We deduce that (as $V_k$ is orthogonal)
\begin{equation}
 \begin{split}
 x^T(U_kU_k^T + t I)^{-1}x &=  x^T V_k^T D_t ^{-1}V_kx= \frac{1}{1+t}u^TU_kU_k^Tu +  \frac{1}{t}(u^\perp)^TV_k^TV_ku^\perp\\
 & =\frac{1}{1+t}\|u\|_2^2 +  \frac{1}{t}\|u^\perp\|_2^2.
 \end{split} \nonumber
\end{equation}

We have 
 \begin{equation}
\omega_k(x,t) =  \frac{1}{1 + \sum_{\ell\neq k} \frac{\nu_\ell(x,t)}{\nu_k(x,t)}} \nonumber
\end{equation}
and, using the Pythagorean identity $\|x\|_2^2 = \|P_{E_k}^\perp(x) \|_2^2+ \| P_{E_k}^\perp(x) -x\|_2^2$,

\begin{equation} \label{eq:convergence_score}
\begin{split}
\frac{\nu_\ell(x,t)}{\nu_k(x,t)}=&  \frac{\pi_\ell\sqrt{(1+t)^{r_k}t^{d-r_k}}}{\pi_k\sqrt{(1+t)^{r_\ell}t^{d-r_\ell}}} \\
&\times \exp\left(-\frac{1}{2}(\frac{1}{1+t}\|P_{E_\ell}^{\perp}(x)\|_2^2+\frac{1}{t}\|P_{E_\ell}^{\perp}(x)-x\|_2^2 - \frac{1}{1+t}\|P_{E_k}^{\perp}(x)\|_2^2-\frac{1}{t}\|P_{E_k}^{\perp}(x)-x\|_2^2) \right) \\
=& \frac{\pi_\ell\sqrt{(1+t)^{r_k}t^{d-r_k}}}{\pi_k\sqrt{(1+t)^{r_\ell}t^{d-r_\ell}}}  \exp\left(-\frac{1}{2}(\frac{1}{t} - \frac{1}{1+t})(\|P_{E_\ell}^{\perp}(x)-x\|_2^2 -\|P_{E_k}^{\perp}(x)-x\|_2^2) \right) \\
&= \frac{\pi_\ell\sqrt{(1+t)^{r_k}t^{d-r_k}}}{\pi_k\sqrt{(1+t)^{r_\ell}t^{d-r_\ell}}}  \exp\left(-\frac{1}{2}(\frac{1}{t(1+t)})(\|P_{E_\ell}^{\perp}(x)-x\|_2^2 -\|P_{E_k}^{\perp}(x)-x\|_2^2) \right) \\
\end{split}
\end{equation}

We deduce, as it is the ratio of (potentially) exponentially decreasing function over the square root of a  polynomial function.

\begin{equation}
\lim_{t \to 0} \frac{\nu_\ell(x,t)}{\nu_k(x,t)} =
\begin{cases}
0 & \text{if } \|P_{E_\ell}^{\perp}(x)-x\|_2^2 > \|P_{E_k}^{\perp}(x)-x\|_2^2, \\
\infty & \text{if } \|P_{E_\ell}^{\perp}(x)-x\|_2^2 < \|P_{E_k}^{\perp}(x)-x\|_2^2.
\end{cases} \nonumber
\end{equation}
Using again the fact that $ \|P_{E_k}^{\perp}(x)-x\|_2^2 = \|x\|_2^2 -  \|P_{E_k}^{\perp}(x)\|_2^2$,  we have equivalently
\begin{equation}\label{eq:conv_proj1}
\lim_{t \to 0} \frac{\nu_\ell(x,t)}{\nu_k(x,t)} =
\begin{cases}
0 & \text{if } \|P_{E_\ell}^{\perp}(x)\|_2^2 < \|P_{E_k}^{\perp}(x)\|_2^2, \\
\infty & \text{if } \|P_{E_\ell}^{\perp}(x)\|_2^2 > \|P_{E_k}^{\perp}(x)\|_2^2.
\end{cases}
\end{equation}
We deduce
\begin{equation}
\lim_{t \to 0} \omega_k(x,t) =
\begin{cases}
1 & \text{if } \forall \ell \neq k, \|P_{E_k}^{\perp}(x)\|_2^2 > \|P_{E_\ell}^{\perp}(x)\|_2^2, \\
0 & \text{if } \exists \ell \neq k, \|P_{E_k}^{\perp}(x)\|_2^2 < \|P_{E_\ell}^{\perp}(x)\|_2^2.
\end{cases} \nonumber
\end{equation}

Now suppose $r_\ell = r$. Let $k,\ell \in L$, we have 
\begin{equation}
\begin{split}
\frac{\nu_\ell(x,t)}{\nu_k(x,t)} 
= \frac{\pi_\ell\sqrt{(1+t)^{r_k}t^{d-r_k}}}{\pi_k\sqrt{(1+t)^{r_\ell}t^{d-r_\ell}}} = \frac{\pi_\ell}{\pi_k}\\
\end{split} \nonumber
\end{equation}
and 
\begin{equation}
\begin{split}
\lim_{t\to 0} \omega_k(x,t)= \frac{1}{ \sum_{\ell \in L} \pi_\ell/\pi_k }=\frac{\pi_k}{\sum_{\ell \in L} \pi_\ell}. \\
\end{split} \nonumber
\end{equation}
Let $k\in L^c$. It implies that there exists $\ell \in L$ such that $\|P_{E_k}^{\perp}(x)\|_2^2 < \|P_{E_\ell}^{\perp}(x)\|_2^2 $ and, as in~\eqref{eq:conv_proj1}, $\lim_{t\to 0} \frac{\nu_\ell(x,t)}{\nu_k(x,t)} =\infty $, and $\lim_{t\to 0}  \omega_k(x,t) = 0$.

\end{proof}

Based on our understanding of $\omega(x,t)$, we can show the convergence to the limiting projection $P_\Sigma^\perp$.

\begin{proof}[Proof of Lemma~\ref{lem:limit_proj_GMM}]
Note that as $t_n\rightarrow 0$, $P^{n}(x) = \frac{1}{1+t_n} \sum_{\ell=1}^K \omega_\ell(x,t_n)P_{E_\ell}^\perp(x)  \to_{n\to\infty}  \sum_{\ell=1}^K \omega_\ell(x,0)P_{E_\ell}^\perp(x) $. If $x$ is strictly closer to  one  subspace $E_k$ (with respect to the orthogonal projection), the previous Lemma shows that $\lim_{n\rightarrow\infty} P^{n}(x) = P_{E_k}^\perp(x)$. This is exactly the definition of the orthogonal projection onto the union-of-subspaces $ \Sigma =\bigcup_\ell E_\ell$, i.e. $ P_{E_k}^\perp(x)= P_\Sigma^\perp(x)$

When the dimensions of all subspaces are equal and $x$ is equidistant to subspaces $(E_\ell)_{\ell\in L}$, by Lemma~\ref{lem:limit_omega}, we have that 
\begin{equation}
\lim_{n\to\infty}  P^{n}(x) =\tilde{P}_\Sigma(x) := \sum_{k\in L}\frac{\pi_k}{\sum_{\ell\in L} \pi_\ell} P_{E_k}^\perp(x). \nonumber
\end{equation} 

\end{proof}

Note that this is a point-wise convergence result. Uniform convergence (for the operator norm) is not verified in general. Indeed, on the frontier between subspaces, $\tilde{P}_\Sigma(x)$ averages the projection over the different subspaces instead of selecting a subspace as $P_\Sigma^\perp$ does.

\subsection{Convergence for the LR-GMM}

We will first prove a local convergence result when iterates lie in a neighborhood of $\hat{x}$ away from the frontier between subspaces, and then obtain global convergence  by showing that the iterates escape the frontier.

\subsubsection{Local convergence} \label{sec:local-convergence-appx}
To discuss our local convergence result, we will focus on regions that lie away from the frontier $F$, which will be defined shortly. For $x \in \mathbb{R}^d$, we denote by $k_x$ an index such that $P_{E_{k_x}}^\perp(x) \in P_\Sigma^\perp(x)$. We define the frontier between subspaces $F$ at $x$ to be 
\begin{align}
    F := \left\{x : \exists \ell \neq k_x,\ \|P_{E_\ell}^{\perp}(x)-x\|_2^2 = \|P_{E_{k_x}}^{\perp}(x)-x\|_2^2\right\} = \left\{x : \#P_\Sigma^\perp(x)> 1 \right\}. \nonumber
\end{align} 
where $\# P_\Sigma^\perp(x)$ denotes the number of elements in $P_\Sigma^\perp(x)$.
When $k_x$ is unique, we define the set $\Omega_{\eta}$ to be \begin{align*}
    \Omega_{\eta} &:=\left\{x : \left| \|P_{E_\ell}^{\perp}(x)-x\|_2^2 - \|P_{E_{k_x}}^{\perp}(x)-x\|_2^2\right| \geq \eta\ \text{for all}\ \ell \neq k_x\right\}\\
    &= \left\{x : \left| \|P_{E_\ell}^{\perp}(x)\|_2^2 - \|P_\Sigma^{\perp}(x)\|_2^2\right| \geq \eta\ \text{for all}\ \ell \neq k_x\right\} .
\end{align*} 
This a set of points that lie away from the frontier $F$: when $\eta=0$, $\Omega_{\eta}$ is the whole space, while when $\eta$ increases, elements $x$ of $\Omega_{\eta}$ are closer to $E_{k_x}$ than other subspaces $E_{\ell}$.
A visual representation of $\Omega_\eta$ and $F$ can be found in  Figure~\ref{fig:omega}.
We  bound the rate of convergence of the operator norm of $P^n-P_\Sigma^\perp$ restricted to $\Omega_{\eta}$ denoted by $\|P^n-P_\Sigma^\perp\|_{\mathrm{op},\Omega_\eta}$.

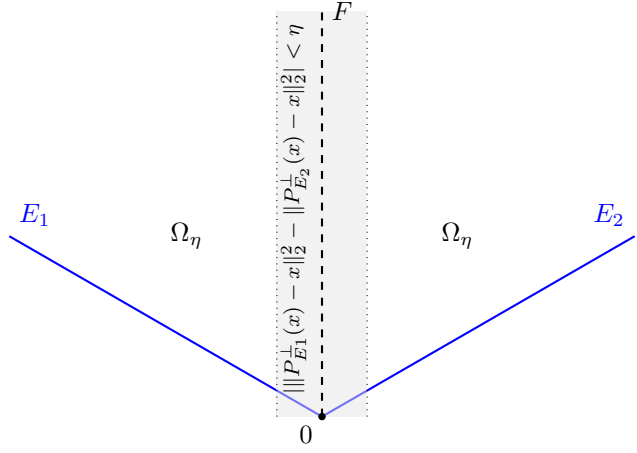
\begin{figure}[!h] \label{fig:omega}
\begin{center}
\begin{tikzpicture}[scale=1.2]
  \draw[blue, thick, rotate around={30:(0,0)}] (0,0) -- (4,0) node[anchor=south east] {\(E_2\)};
  \draw[blue, thick, rotate around={-30:(0,0)}] (-4,0)  node[anchor=south west] {\(E_1\)} -- (0,0);

  \fill[gray!20, opacity=0.5] (-0.5,0) rectangle (0.5,4.5);
  \draw[dotted] (-0.5,0) -- (-0.5,4.5);
  \draw[dotted] (0.5,0) -- (0.5,4.5);
  \node[rotate=90,font=\footnotesize] at (-0.3,2.3) {\(|\|P_{E1}^{\perp}(x)-x\|_2^2 - \|P_{E_{2}}^{\perp}(x)-x\|_2^2|< \eta\)};

  \draw[black, thick, dashed] (0,0) -- (0,4.5) node[anchor=west] {\(F\)};
  
  \node at (-1.5,2) {\(\Omega_\eta\)};
  \node at (1.5,2) {\(\Omega_\eta\)};
  
  \filldraw (0,0) circle (1pt) node[below left] {$0$};
\end{tikzpicture}
\end{center}
\caption{A representation of the set $\Omega_\eta$. Note that boundary between $\Omega_\eta$ and $\Omega_\eta^c$ might not be linear.}
\end{figure}

\begin{lem} \label{lem:bound_op_norm1}
Consider $\Sigma = \bigcup_{\ell=1}^K E_\ell$ with $\dim(E_\ell) =r$ for all $\ell$. Consider $P^n$ defined by~\eqref{def:Pn_GMM}. Let $\eta >0$ and consider $x \in \Omega_{\eta}$. Then we have 
\begin{equation}
\begin{split}
&\frac{\|P^n(x) -P_\Sigma^\perp(x)\|_2}{\|x\|_2} \leq  \|P^n-P_\Sigma^\perp\|_{\mathrm{op},\Omega_\eta} \leq 2  \sum_{\ell=1, \ell \neq k_x}^K \frac{\pi_\ell}{\pi_{k_x}}  \exp\left(-\frac{1}{2}\frac{\eta}{t_n(1+t_n)}\right) + t_n. \\
\end{split} \nonumber
\end{equation} As a consequence, for $t_n = \sigma_n^2$, \begin{equation}
\begin{split}
 \|P^n-P_\Sigma^\perp\|_{\mathrm{op},\Omega_\eta} &= O (e^{-\frac{\eta}{4\sigma_n^2} } +\sigma_n^2  ) .
\end{split} \nonumber
\end{equation}
\end{lem}

\begin{proof}
 Consider $t=t_n \in [0,1]$. Let $x \in \Omega_\eta$ and consider the unique $k \in \{1,\ldots,K\}$ such that $P_\Sigma(x) = P_{E_k}^\perp(x)$. We have, for $\ell \neq k$, $\|P_{E_\ell}^{\perp}(x)-x\|_2^2 -\|P_{E_k}^{\perp}(x)-x\|_2^2 \geq \eta$ and, using \eqref{eq:convergence_score},
\begin{equation} \label{eq:bound_op_norm1_eq1}
\begin{split}
\frac{\nu_\ell(x,t)}{\nu_k(x,t)} &= \frac{\pi_\ell}{\pi_k} \exp\left(-\frac{1}{2}\frac{1}{t(1+t)}(\|P_{E_\ell}^{\perp}(x)-x\|_2^2 -\|P_{E_k}^{\perp}(x)-x\|_2^2) \right) \\
& \leq \frac{\pi_\ell}{\pi_k} \exp\left(-\frac{1}{2}\frac{\eta}{t(1+t)} \right). \\
\end{split}
\end{equation}
 Using the expression of $P^n$ from \eqref{eq:expr_P_n}, with the triangle inequality, the fact that $\|P_{E_\ell}^\perp(x)\|_2 \leq \|x\|_2$ for any $x$ and $\frac{1}{1+t} \leq 1$, we deduce that
\begin{equation}
\begin{split}
\|P^n(x) -P_\Sigma^\perp(x)\|_2 &= \big\|(\frac{1}{1+t}\omega_k(x,t)-1) P_{E_k}^\perp(x) + \frac{1}{1+t}\sum_{\ell=1;\ell \neq k}^K \omega_\ell(x,t) P_{E_\ell}^\perp(x)\big\|_2\\  
&\leq \|x\|_2\left(\left|\frac{1}{1+t}\omega_k(x,t)-1\right| + \frac{1}{1+t}\sum_{\ell=1; \ell \neq k}^K \omega_\ell(x,t)\right) \\
&\leq  \|x\|_2 \left(\left|\frac{\omega_k(x,t)-1 }{1+t}\right| + \frac{t}{1+t} + \frac{1}{1+t}\sum_{\ell=1 ;\ell \neq k}^K \omega_\ell(x,t)\right) \\
&\leq  \|x\|_2 \left(\left|\omega_k(x,t)-1 \right| + t + \sum_{\ell=1; \ell \neq k}^K \omega_\ell(x,t)\right). \\
\end{split} \nonumber
\end{equation}
We have, with~\eqref{eq:bound_op_norm1_eq1},
\begin{equation}
\begin{split}
|\omega_k(x,t)-1|  &= \left| \frac{\nu_k(x,t)}{\sum_{\ell=1 }^K\nu_\ell(x,t) } -1\right|\\
&=\left|\frac{\sum_{ \ell=1 ;\ell \neq k}^K\nu_\ell(x,t) }{\sum_{\ell=1 }^K\nu_\ell(x,t)}\right| \leq\sum_{\ell=1 ;\ell \neq k}^K \frac{\nu_\ell(x,t)}{\nu_k(x,t)} \leq   \sum_{\ell=1 ;\ell \neq k}^K \frac{\pi_\ell}{\pi_k}  \exp\left(-\frac{1}{2}\frac{\eta}{t(1+t)} \right).
\end{split} \nonumber
\end{equation}

For $\ell \neq k$, we have 
\begin{equation}
\begin{split}
\omega_\ell(x,t)  &=  \frac{\nu_\ell(x,t)}{\sum_{\ell'=1 }^K\nu_{\ell'}(x,t) }\leq \frac{\nu_\ell(x,t)}{\nu_{k}(x,t) } \leq \frac{\pi_\ell}{\pi_k}  \exp\left(-\frac{1}{2}\frac{\eta}{t(1+t)} \right)  .\\
\end{split} \nonumber
\end{equation}

We deduce overall 
\begin{equation}
\begin{split}
\frac{\|P^n(x) -P_\Sigma^\perp(x)\|_2}{\|x\|_2} &\leq 2  \sum_{\ell=1,;0\ell \neq k}^K \frac{\pi_\ell}{\pi_k}  \exp\left(-\frac{1}{2}\frac{\eta}{t(1+t)}\right)   + t.  \\
\end{split} \nonumber
\end{equation}

\end{proof}
Using this bound, we establish local convergence guarantees: there exists a neighborhood around $\hat{x}$ such that any iterate entering this region remains within it and converges at the following rate.

\begin{lem} \label{lem:stability_local}
Consider the hypotheses of Lemma~\ref{lem:bound_op_norm1}. Suppose $\delta \beta < 1$ with $\delta =\delta(\mu A^T  A), \beta = \beta_\Sigma(P_\Sigma^\perp)$. Let $C>0$ and $x \in B(\hat{x},C) $.  Suppose $B(\hat{x},C) \subset \Omega_{\eta}$.  Then there is $n_0$ such that for $n \geq n_0$, defining $z = P^n(x) - \mu A^TA  (P^n(x) -\hat{x})$, we have $z \in B(\hat{x},C)$ .
\end{lem}
\begin{proof}
We have
\begin{equation} 
\begin{split}
 \|z-\hat{x}\|_2 &\leq  \delta \beta  \| x -\hat{x}  \|_2 + \|I-\mu A^TA\|_{\mathrm{op}} \|P^n(x) -P_\Sigma^\perp(x)) \|_2 \\
 &\leq  \delta \beta  C + \|I-\mu A^TA\|_{\mathrm{op}} \left( 2C'  \sum_{\ell=1, \ell \neq K}^K \frac{\pi_\ell}{\pi_k}  \exp\left(-\frac{1}{2}\frac{c}{t(1+t)}\right)   + C'\frac{t}{1+t}  \right)\\
 \end{split} \nonumber
\end{equation}
Hence, as $\delta\beta <1$, there is $t = \sigma_n^2$ sufficiently small such that $\|z - \hat{x}\|_2 \leq C$.  
\end{proof}
We have just guaranteed the stability of iterates in $ B(\hat{x},C)\subset \Omega_\eta$. Using the previous results, we can now provide a convergence rate if the iterates fall in a basin of attraction of $\hat{x}$. Note that we require $\hat{x}$ to belong to only one subspace $E_k \subset \Sigma$. In practice this is not a problem since the set of intersections of subspaces is of measure $0$ (being of dimension strictly lower than $r = \dim (E_k) $, $k\in \{1,\ldots,K \}$).
\begin{thm}\label{th:local_convergence}
Let $\hat{x} \in \Sigma$ such that there exists a unique $k\in \{1,\ldots,K \}$, such that $\hat{x} \in E_k$. Suppose $\sigma_n^2$ decreases to $0$. There exists $c,C,C_1,C_2$, $n_0$ such that, if there exists $n_1 \geq n_0$  such that  $x_{n_1} \in B(\hat{x},C) $, then, for all $n  \geq 0$
\begin{align*}
\|x_{n_1+n}-\hat{x}\|_2 \leq\ C_1(\sqrt{\delta \beta})^n & \|x_{n_1}-\hat{x}\|_2  + C_2 ( e^{-\frac{c}{\sigma_{n/2}^2} } +\sigma_{n/2}^2  ).
\end{align*}
\end{thm}

\begin{proof}
As $\hat{x} \neq 0$, taking $C$ sufficiently small, we have $B(\hat{x},C) \subset \Omega_{\eta}$. With Lemma~\ref{lem:stability_local}, there is $n_0$ such that for $n_1 \geq n_0$,  we have $x_{n_1+1}\in B(\hat{x},C)$. By induction, for any $n \geq 0$, $x_{n_1+n}\in B(\hat{x},C) \subset \Omega_\eta$, we deduce that $x_n$ is bounded, and with Theorem~\ref{th:conv_var_proj} and the bound on $\|P^{l} -P_\Sigma\|_{\mathrm{op},\Omega_\eta}$ from Lemma~\ref{lem:bound_op_norm1}, we have that there exists $C'$,
\begin{equation}
\begin{split}
\|x_{n_1+n}-\hat{x}\|_2 &\leq C'( (\sqrt{\delta \beta})^n \|x_{n_1}-\hat{x}\|_2  +  \max_{l=\lfloor n/2\rfloor,n}\|P^{l}(x_l) -P_\Sigma(x_l) \|_2 ) \\
&\leq  C'(\sqrt{\delta \beta})^n \|x_{n_1}-\hat{x}\|_2 + O( e^{-\frac{c}{\sigma_{n/2}^2} } +\sigma_{n/2}^2  ) .
\end{split} \nonumber
\end{equation}
\end{proof}

\subsection{Global convergence}
In this Section we show global convergence of $(x_n)$. We guarantee that after a given ``burn-in'' time period, the iterates $x_n$ converge linearly to $\hat{x} \in \Sigma$ (with the right noise schedule)  under the condition that $\delta\beta <1$. The main technical difficulty of the proof is when $x_n$ is potentially close to the frontiers between subspaces $F_L = \{x: \|P_\Sigma(x)\|_2 = \|P_{E_\ell}^\perp(x)\|_2, \ell \in L \}$, $L\subset \{1,\ldots,K\}$ (see Figure~\ref{fig:distance-to-subspaces}). Note that we define $L^c := \{1,\ldots,K\} \setminus L$. In this case, we use the fact that the iterations pull the iterates away from this frontier. 

\begin{lem} \label{lem:bound_iterates_F} 
Consider $\Sigma =\bigcup_{\ell=1}^K E_\ell$ and $P^n$ defined in~\eqref{def:Pn_GMM}. Let $\delta:= \delta(\mu A^TA)$. Let  $\beta = \beta_\Sigma(P_\Sigma^\perp)$. Let $L\subset \{1,\ldots,K\}$ and $v \in F_L$. Then we have 
\begin{equation}
\begin{split}
    \|x_{n+1} -\hat{x}\|_2 
     &\leq \delta \beta  \| x_n-\hat{x}\|_2  + \|v\|_2  \delta \left(\sum_{\ell \in L^c } \omega_{\ell}(x_n,t_n) \right)+ \delta(1  +\beta) \| x_n-v\|_2 +\delta t_n\|\hat{x}\|_2.\\
\end{split} \nonumber
\end{equation}
\end{lem}

\begin{proof}
 Consider $t= t_n>0$. Letting $v \in F_L$, we have, using the fact that $ P^n(x) = \frac{1}{1+t}\sum_{\ell =1}^K \omega_{\ell}(x,t) P_{E_\ell}^\perp (x)$ for any $x$ and  $\sum_{\ell } \omega_\ell(x_n,t) =1$, by convexity,

\begin{equation}
\begin{split}
    \|x_{n+1} -\hat{x}\|_2 &=  \| (I-\mu A^TA) (P^n(x_n)-\hat{x}) \|_2 \\
    &= \| (I-\mu A^TA) ( \sum_{\ell=1}^{K} \frac{\omega_{\ell}(x_n,t)}{1+t} \left(P_{E_\ell}^\perp (x_n) - P_{E_\ell}^\perp (v)+P_{E_\ell}^\perp (v)\right)-\hat{x}) \|_2 \\
    &\leq   \sum_{\ell=1 }^K \frac{\omega_{\ell}(x_n,t)}{1+t} \| (I-\mu A^TA) (P_{E_\ell}^\perp (v)- (1+t)\hat{x}  )\|_2   + \| (I-\mu A^TA) (  \sum_{\ell=1 }^K\frac{\omega_{\ell}(x_n,t)}{1+t} (P_{E_\ell}^\perp (x_n)- P_{E_\ell}^\perp (v))\|_2\\
     &\leq   \sum_{\ell=1 }^K\frac{\omega_{\ell}(x_n,t)}{1+t}  \| (I-\mu A^TA) (P_{E_\ell}^\perp (v)-\hat{x}- t\hat{x} )\|_2    + \sum_{\ell=1}^K \frac{\omega_{\ell}(x_n,t)}{1+t} \| (I-\mu A^TA ) P_{E_\ell}^\perp (x_n-v))\|_2.\\
\end{split} \nonumber
\end{equation}
With the triangle inequality and the RIC, and the fact that $1+t \geq 1$ 
 \begin{equation}
\begin{split}   
   \|x_{n+1} -\hat{x}\|_2   &\leq \delta \sum_{\ell=1}^K \frac{\omega_{\ell}(x_n,t)}{1+t} (\|P_{E_\ell}^\perp (v)-\hat{x}\|_2 +t\|\hat{x}\|)   + \sum_{\ell=1}^K \frac{\omega_{\ell}(x_n,t)}{1+t}  \delta  \|P_{E_\ell}^\perp (x_n-v)\|_2\\
    &\leq \delta \sum_{\ell=1}^K   \frac{\omega_{\ell}(x_n,t)}{1+t} \|P_{E_\ell}^\perp (v)-\hat{x}\|_2   + \sum_{\ell=1}^K \frac{\omega_{\ell}(x_n,t)}{1+t} \delta  \| x_n-v\|_2 +\frac{\delta t}{1+t} \|\hat{x}\|_2\\
      &\leq \delta \sum_{\ell=1}^K   \omega_{\ell}(x_n,t)\|P_{E_\ell}^\perp (v)-\hat{x}\|_2   + \delta  \| x_n-v\|_2+\delta t \|\hat{x}\|_2.\\
\end{split} \nonumber
\end{equation}
As  $P_{E_\ell}^\perp(v) \in P_\Sigma^\perp(v)$ for $\ell \in L$ by definition of $F_L$, using the restricted Lipschitz condition of $P_\Sigma^\perp$, and the triangle inequality, we have 
\begin{equation}
\begin{split}
    \|x_{n+1} -\hat{x}\|_2 
    &\leq  \delta \beta  \sum_{\ell \in L } \omega_{\ell}(x_n,t) \|v-\hat{x}\|_2  + \delta\sum_{\ell \in L^c } \omega_{\ell}(x_n,t) \|P_{E_\ell}^\perp (v)-\hat{x}\|_2  +\delta \| x_n-v\|_2+\delta t \|\hat{x}\|_2\\
     &=  \delta \beta  \| v-\hat{x}\|_2  + \delta\sum_{\ell \in L^c }  \omega_{\ell}(x_n,t) (\|P_{E_\ell}^\perp (v)-\hat{x}\|_2 -\beta \|v-\hat{x}\|_2) +  \delta \| x_n-v\|_2+\delta t\|\hat{x}\|_2\\
      &\leq \delta \beta\|x_n-\hat{x}\|_2   + \delta\sum_{\ell \in L^c }  \omega_{\ell}(x_n,t) (\|P_{E_\ell}^\perp (v)-\hat{x}\|_2 -\beta \|v-\hat{x}\|_2) + ( \delta +\delta\beta) \| x_n-v\|_2+\delta t\|\hat{x}\|_2.\\
\end{split} \nonumber
\end{equation}
We remark that $\|P_{E_\ell}^\perp (v)-\hat{x}\|_2 -\beta \|v-\hat{x}\|_2 \leq \|(I- P_{E_\ell}^\perp) (v)\|_2 +(1-\beta )\|v-\hat{x}\|_2  \leq \|(I- P_{E_\ell}^\perp) (v)\|_2  \leq \|v\|_2$ and we conclude
\begin{equation}
\begin{split}
    \|x_{n+1} -\hat{x}\|_2 
 &\leq  \delta \beta  \| x_n-\hat{x}\|_2  + \delta\sum_{\ell \in L^c } \omega_{\ell}(x_n,t)\|v\|_2  + ( \delta +\delta\beta) \| x_n-v\|_2+\delta t\|\hat{x}\|_2 .\\
\end{split} \nonumber
\end{equation}

\end{proof}

We need the following Lemma to bound the quantities  $\omega_\ell(x_n,t)$ given in Lemma~\ref{lem:bound_iterates_F}.
\begin{lem} \label{lem:bound_tail_omega}
Let  $L\subset \{1,\ldots,K \}$, $x \in \mathbb{R}^d$ such that $ P_{E_k}^\perp(x) \in P_\Sigma^\perp(x)$. Suppose for $ \ell \in L^c$ we have $\|P_{E_\ell}^\perp(x_n)-x_n\|_2^2- \|P_{E_k}^\perp(x_n)-x_n\|_2^2 \geq \alpha$, then 
\begin{equation} 
\sum_{\ell\in L^c}\omega_\ell(x,t)  \leq\left( \sum_{\ell\in L^c}\frac{\pi_\ell}{\pi_k}\right) \exp\left(-\frac{1}{2}\frac{\alpha}{t(1+t)} \right) \nonumber
\end{equation}
\end{lem}
\begin{proof}
We have, similarly as in the proof of Lemma~\ref{lem:bound_op_norm1}, for $\ell \in L^c$,
\begin{equation}
\begin{split}
\omega_\ell(x,t) &\leq \frac{\pi_\ell}{\pi_k} \frac{\nu_{\ell}(x,t)}{\nu_{k}(x,t)} \leq 
 \frac{\pi_\ell}{\pi_k}  \exp\left(-\frac{1}{2}\frac{\alpha}{t(1+t)} \right).\\
\end{split} \nonumber
\end{equation} 
We deduce 
\begin{equation} 
\begin{split}
\sum_{\ell\in L^c}\omega_\ell(x,t)  & \leq \sum_{\ell\in L^c} \frac{\pi_\ell}{\pi_k} \exp\left(-\frac{1}{2}\frac{\alpha}{t(1+t)} \right).  \\ \nonumber
\end{split}
\end{equation}
\end{proof}

Suppose $P_{E_k}^\perp(x_n) \in P_{\Sigma}(x_n)  $. Let $L \subset  \{1,\ldots,K\}$. Define 
\begin{equation}
\eta_n^L := \min_{\ell \in L, \ell \neq k}\left(\|P_{E_\ell}^\perp(x_n)-x_n\|_2^2- \|P_{E_k}^\perp(x_n)-x_n\|_2^2 \right)  = \min_{\ell \in L,\ell \neq k}\left(\|P_\Sigma(x_n)\|_2^2-\|P_{E_\ell}^\perp(x_n)|_2^2\right)  \nonumber
\end{equation}
Before giving the main convergence result, we need to link the definition of $\eta_n^{L_n}$
 with the accumulation of points near $F_{L_n}$, where $L_n$ is a sequence of indices defining a sequence of frontier sets.
 
 \begin{lem} \label{lem:continuity_argument}
Consider a sequence $(x_n)$ and a sequence of indices $L_n\subset \{1,\ldots,K\}$. Suppose $\lim_{n\to \infty}\eta_n^{L_n} = 0$. Then $d(x_n,F_{L_n}) = \min_{v \in F_{L_n} }\|v-x_n\|_2\to_{n\to \infty} 0$.

\end{lem}

 \begin{proof}
Define $k_x$ the index of a subspace closest to $x$.  Consider the functions 
\begin{equation}
f_{L_n}(x) = \min_{\ell \in L_n}(\|P_{E_\ell}^\perp(x)-x\|_2^2- \|P_{E_{k_x}}^\perp(x)-x\|_2^2 )\geq 0 . \nonumber
\end{equation}
We have that $f_{L_n}$ is continuous. Indeed it is the minimum of $|L_n|$ continuous functions $x \mapsto \|P_{E_\ell}^\perp(x)-x\|_2^2- \|P_{E_{k_x}}^\perp(x)-x\|_2^2 = \|P_{\Sigma}^\perp(x)\|_2^2-\|P_{E_\ell}^\perp(x)\|_2^2$ (because $ x \mapsto \|P_{\Sigma}^\perp(x)\|_2^2$ is continuous, see \cite[Lemma 3.7]{traonmilin2024towards}).

Hence, for a fixed $n$, any sequence $x_i$ such that $\lim_{i\to \infty}f_{L_n}(x_i) = 0$ is such that $d(x_i,F_{L_n}) \to_{i\to\infty} 0$ as $F_{L_n} =\{ x : f_{L_n}(x)=0\}$. As there is only a finite number of functions $f_{L_n}$, the sequence $f_{L_n}(x_n)$ can be partitioned into a finite number of subsequences where $L_n = L$. For each of these subsequences, we have $\lim_{n\to \infty} d(x_n,F_{L}) = 0$, which concludes the proof. 
\end{proof}

We now give the main Theorem controlling global convergence. 

\begin{proof}[Proof of Theorem~\ref{th:global_convergence_gmm}]
 First, we use Lemma~\ref{lem:bound_iterates_F}  with $v = 0$ and $L = \{1,\ldots,K\}$ this gives 
\begin{equation}
\begin{split}
    \|x_{n+1} -\hat{x}\|_2  &\leq \delta \beta  \| x_n-\hat{x}\|_2  +   \delta(1+\beta)\| x_n\|_2+\delta t\|\hat{x}\|_2.\\
\end{split} \nonumber
\end{equation}
 We observe that there exists $b>0$ and a rank $n_0$ such that for $n\geq n_0$, $\|x_n\|_2>b$, otherwise we can find a subsequence of $x_n$ converging to $0$ and, as $t \to 0$ taking the above inequality to the limit leads to  $\|\hat{x}\|_2 \leq  \delta \beta  \|\hat{x}\|_2  $ which contradicts the hypothesis $\delta\beta< 1$.

At each iteration, we can use two bounds: 
\begin{itemize}
    \item Lemma~\ref{lem:bound_iterates_F} guarantees that for any choice of $L \subset \{1, \ldots, K\} $, $v \in F_L$, we have, as $\sum_{\ell \in L^c } \omega_{\ell}(x_n,t) \leq 1$,
\begin{equation}
\begin{split}
    \|x_{n+1} -\hat{x}\|_2 
     &\leq \delta \beta  \| x_n-\hat{x}\|_2  + \|v\|_2  \delta \left(\sum_{\ell \in L^c } \omega_{\ell}(x_n,t) \right)+ \delta(1+\beta) \| x_n-v\|_2+\delta t\|\hat{x}\|_2\\
      &\leq \delta \beta  \| x_n-\hat{x}\|_2  + \|x_n +v -x_n\|_2  \delta \left(\sum_{\ell \in L^c } \omega_{\ell}(x_n,t) \right)+ \delta(1+\beta) \| x_n-v\|_2+\delta t\|\hat{x}\|_2\\
     &\leq \delta \beta  \| x_n-\hat{x}\|_2  + \|x_n\|_2  \delta \left(\sum_{\ell \in L^c } \omega_{\ell}(x_n,t) \right)+ \delta(2+\beta) \| x_n-v\|_2+\delta t\|\hat{x}\|_2;\\
\end{split} \nonumber
\end{equation}

\item Lemma~\ref{lem:bound_iterates_gen2} gives that between $x_n$ and $x_{n+1}$, we have the bound
\begin{equation}
\begin{split}
    \|x_{n+1} -\hat{x}\|_2 
     &\leq  \delta \beta  \| x_n-\hat{x}\|_2  +\|I-\mu A^T A\|_{\mathrm{op}} \| P^n(x_n)-P_\Sigma(x_n)\|_2. \\
\end{split}\nonumber
\end{equation}
\end{itemize}
This gives
\begin{equation}
\begin{split}
    \|x_{n+1} -\hat{x}\|_2 
     &\leq  \delta \beta  \| x_n-\hat{x}\|_2  
     + D \min (\| P^n(x_n)-P_\Sigma(x_n)\|_2 , \|x_n\|_2 ( \sum_{\ell \in L^c } \omega_{\ell}(x_n,t) ) + \|v-x_n\|_2 + t))\\
\end{split}\nonumber
\end{equation} 
where $D =\max(\|I-\mu A^T A\|_{\mathrm{op}},\delta(2+\beta), \delta \|\hat{x}\|_2) $.
 
We can specifically choose $v_n \in F_{L_n}$ at each step $n$ where $L_n$ and $v_n$ are to be set later. We have 
\begin{equation}
\begin{split}
    \|x_{n+1} -\hat{x}\|_2 
     & \leq  \delta \beta  \| x_n-\hat{x}\|_2 +\|x_n\|_2 D U_n
\end{split}\nonumber
\end{equation}
where $U_n = \min \left(\frac{\| P^n(x_n)-P_\Sigma(x_n)\|_2}{\|x_n\|_2}  , ( \sum_{\ell \in L_n^c } \omega_{\ell}(x_n,t) ) + \frac{\|v_n-x_n\|_2}{\|x_n\|}+\frac{t}{\|x_n\|_2}\right)$.

Define $[K] =\{1,\ldots,K\}$ and $k_n$ such that $P_{E_{k_n}}(x_n) \in P_\Sigma^\perp(x_n)$.  Let $\ell_n \in \arg \min_{\ell\neq k_n} \|P_{E_\ell}^\perp(x_n)-x_n\|_2^2- \|P_{E_{k_n}}^\perp(x_n)-x_n\|_2^2$ , consider $L_n \supset \{k_n,\ell_n\}$. 

We have, using Lemma~\ref{lem:bound_op_norm1}, Lemma~\ref{lem:bound_tail_omega} and the fact that $\|x_n\|\geq b >0$ for $n\geq n_0$, 
\begin{equation}
\begin{split}
  U_n &= O\left(  \min \left(\exp( -\frac{1}{4} \frac{\eta_n^{[K]}}{t} )  +t ,  \exp\left(-\frac{1}{4}\frac{\eta_n^{L_n^c}}{t} \right) +\frac{\|x_n-v_n\|_2}{\|x_n\|_2}+\frac{t}{\|x_n\|_2}\right)\right)\\
     & =  O \left( \min \left(\exp( -\frac{1}{4} \frac{\eta_n^{[K]}}{t} )  , \exp\left(-\frac{1}{4}\frac{\eta_n^{L_n^c}}{t} \right)+\|x_n-v_n\|_2\right) + t\right),\\
\end{split}\nonumber
\end{equation}
i.e. there exists $D' > 0$ such that 
\begin{equation}
\begin{split}
  U_n&\leq   D' (\min (\exp( -\frac{1}{4} \frac{\eta_n^{[K]}}{t} )   ,  \exp\left(-\frac{1}{4}\frac{\eta_n^{L_n^c}}{t} \right)+\|x_n-v_n\|_2) +t).\\
\end{split}\nonumber
\end{equation}
Let $c'> 0$, we distinguish two cases:
\begin{itemize}
    \item if $\exp( -\frac{1}{4} \frac{\eta_n^{[K]}}{t} ) \leq c't$  we have 
\begin{equation}
\begin{split}
   U_n &\leq D'(c' t  +t) = D'(1+c')t; \\
\end{split}\nonumber
\end{equation}

\item Otherwise, we have $\exp( -\frac{1}{4} \frac{\eta_n^{[K]}}{t} ) > c't$,  $ \eta_n^{[K]} \leq 4t \log(\frac{1}{c't}) \to_{t\to 0} 0$. Subsequently we choose 
\begin{equation}
L_n= \left\{ \ell : \exp\left(-\frac{1}{4}\frac{\|P_{E_\ell}^\perp(x_n)-x_n\|_2^2- \|P_{E_{k_n}}^\perp(x_n)-x_n\|_2^2}{t} \right) > c't \right\}\nonumber
\end{equation}
which contains $\ell_n$ (by definition of $\ell_n$ and using the condition $\exp( -\frac{1}{4} \frac{\eta_n^{[K]}}{t} ) > c't$). 
This guarantees that  $ \exp\left(-\frac{1}{4}\frac{\eta_n^{L_n^c}}{t} \right) \leq c't$ and
 \begin{equation}
\begin{split}
   U_n &\leq  D'(c't+\|x_n-v_n\|_2+t) = D' ( (1+c')t + \|x_n-v_n\|_2) \\
\end{split}\nonumber
\end{equation}
 Notice that $\eta_n^{L_n} =\eta_n^{[K]} \to_{n\to \infty} 0$, taking $v_n \in \arg\min_{v \in F_{L_n}} \|v-x_n\|_2$, with Lemma~\ref{lem:continuity_argument}, we have   $\lim_{n\to\infty}\|v_n-x_n\|_2 =0$.
   
\end{itemize}
As $t \to 0$, we deduce that any subsequence of $U_n$ converges to $0$ and $\lim_{n\to \infty} U_n = 0 $.

We have 
\begin{equation}
 \|x_{n+1}-\hat{x}\|_2 \leq \delta\beta \|x_n-\hat{x}\| + D \|x_n\|U_n \leq  (\delta\beta +DU_n) \|x_n-\hat{x}\| + D \|\hat{x}\|U_n \nonumber
\end{equation}
with $\lim_{n\to \infty} U_n = 0$. For sufficiently large $n$,  $\delta\beta +DU_n  \leq \rho < 1$  and as $U_n$ is bounded , e.g. $D \|\hat{x}\|U_n  \leq D''$, for some constant $D''$,we have  $\|x_{n+1}-\hat{x}\|_2 \leq \rho\|x_n-\hat{x}\| + D'' $ and 
\begin{equation}
 \|x_{n+1}-\hat{x}\|_2 \leq \rho^n\|x_0-\hat{x}\|_2 + D'' \frac{1}{1-\rho}.\nonumber
\end{equation}
We deduce that $\|x_n\|_2$ is bounded and $V_n =  D \|x_n\|U_n  \to_{n\to \infty} 0$.

Similarly to Corollary~\ref{th:local_convergence}, we deduce that $\lim_{n\to \infty}\|x_{n+1}-\hat{x}\|_2 = 0$ and that there exists  $n_1$ large enough such that $n\geq n_1$  implies $x_n \in B(\hat{x},C) \subset \omega_{\eta}$ for some $\eta> 0$ and
\begin{equation} \label{eq:bound_Un}
\begin{split}
  U_n&= O( \exp( -\frac{1}{4} \frac{\eta}{t} )  +t).\\
\end{split}
\end{equation}
We have, by induction as in Lemma~\ref{lem:bound_iterates_gen2},  
\begin{equation}
\begin{split}
 \|x_{n}-\hat{x}\|_2 &\leq \delta\beta \|x_n-\hat{x}\|_2 + V_n  \\
  & \leq  (\delta\beta)^n \|x_n-\hat{x}\|_2 +  \sum_{l=0}^{n-1} (\delta\beta)^{n-l-1}V_l
 \end{split}\nonumber
\end{equation}
We deduce as in Theorem~\ref{th:conv_var_proj}, that there is $C'>0$ such that
\begin{equation}
\begin{split}
 \|x_{n}-\hat{x}\|_2 &\leq C'( ( \sqrt{\delta \beta})^{n}   + \max_{l=\lfloor n/2\rfloor,n}V_l). \\
 \end{split}\nonumber
\end{equation}

Using the fact that $V_n = O(U_n)$ and~\eqref{eq:bound_Un}, we deduce that there exists $C>0$ such that
\begin{equation}
\begin{split}
 \|x_{n}-\hat{x}\|_2 &\leq C( ( \sqrt{\delta \beta})^{n}   + \max_{l=\lfloor n/2\rfloor,n}( \exp( -\frac{1}{4} \frac{\eta}{t_l} )  +t_l)\\
 \end{split}\nonumber
\end{equation}

Setting $c= \frac{1}{4} \eta$ and recalling that $ t_l = \sigma^2_l$ leads to the result.

\end{proof}


\end{document}